\documentclass{article}

\usepackage[preprint,nonatbib]{neurips_2020}

\usepackage[utf8]{inputenc} % allow utf-8 input
\usepackage[T1]{fontenc}    % use 8-bit T1 fonts
\usepackage[bookmarks=false]{hyperref}       % hyperlinks
\usepackage{url}            % simple URL typesetting
\usepackage{booktabs}       % professional-quality tables
\usepackage{amsfonts}       % blackboard math symbols
\usepackage{nicefrac}       % compact symbols for 1/2, etc.
\usepackage{microtype}      % microtypography
\usepackage{tablefootnote}  % tablefootnote
\usepackage{xspace}
\usepackage{enumitem}
\usepackage{amsmath}
\usepackage{amsthm}
\usepackage{algorithm}
\usepackage{algpseudocode}
\usepackage{graphicx}
\usepackage{subfigure}
\newcommand{\name}{SnapBoost\xspace}
\DeclareMathOperator*{\argmin}{arg\,min}
\DeclareMathOperator*{\argmax}{arg\,max}
\theoremstyle{plain}
\newtheorem{theorem}{Theorem}
\newtheorem{lemma}[theorem]{Lemma}
\newtheorem{assumption}{Assumption}
\theoremstyle{definition}
\newtheorem{definition}{Definition}

\algblock{ParFor}{EndParFor}
\algnewcommand\algorithmicparfor{\textbf{parfor}}
\algnewcommand\algorithmicpardo{\textbf{do}}
\algnewcommand\algorithmicendparfor{\textbf{end\ parfor}}
\algrenewtext{ParFor}[1]{\algorithmicparfor\ #1\ \algorithmicpardo}
\algrenewtext{EndParFor}{\algorithmicendparfor}

\title{\name{}: A Heterogeneous Boosting Machine}
 
\author{
    Thomas Parnell$^{*}$ \\
    IBM Research -- Z\"urich \\
    Z\"urich, Switzerland \\
    \texttt{tpa@zurich.ibm.com}
    \And
    Andreea Anghel\thanks{equal contribution.} \\
    IBM Research -- Z\"urich \\
    Z\"urich, Switzerland \\
    \texttt{aan@zurich.ibm.com}
    \AND
    Ma{{\l}}gorzata {{\L}}azuka \\
    ETH Z\"urich \\
    Z\"urich, Switzerland \\
    \texttt{lazukam@student.ethz.ch}
    \And
    Nikolas Ioannou \\
    IBM Research -- Z\"urich \\
    Z\"urich, Switzerland \\
    \texttt{nio@zurich.ibm.com} 
    \And
    Sebastian Kurella \\
    ETH Z\"urich \\
    Z\"urich, Switzerland \\
    \texttt{kurellas@student.ethz.ch} 
    \AND
    Peshal Agarwal \\
    ETH Z\"urich \\
    Z\"urich, Switzerland \\
    \texttt{agarwalp@student.ethz.ch}
    \And
    Nikolaos Papandreou \\
    IBM Research -- Z\"urich \\
    Z\"urich, Switzerland \\
    \texttt{npo@zurich.ibm.com}
    \And
    Haralampos Pozidis \\
    IBM Research -- Z\"urich \\
    Z\"urich, Switzerland \\
    \texttt{hap@zurich.ibm.com} 
}

\begin{document}

\maketitle

\begin{abstract}

  Modern gradient boosting software frameworks, such as XGBoost and LightGBM, implement Newton descent in a functional space. 
  At each boosting iteration, their goal is to find the base hypothesis, selected from some base hypothesis class, that is closest to the Newton descent direction in a Euclidean sense.
  Typically, the base hypothesis class is fixed to be all binary decision trees up to a given depth.
  In this work, we study a Heterogeneous Newton Boosting Machine (HNBM) in which the base hypothesis class may vary across boosting iterations.
  Specifically, at each boosting iteration, the base hypothesis class is chosen, from a fixed set of subclasses, by sampling from a probability distribution.
  We derive a global linear convergence rate for the HNBM under certain assumptions, and show that it agrees with existing rates for Newton's method when the Newton direction can be perfectly fitted by the base hypothesis at each boosting iteration.
  We then describe a particular realization of a HNBM, \name{}, that, at each boosting iteration, randomly selects between either a decision tree of variable depth or a linear regressor with random Fourier features.
  We describe how \name{} is implemented, with a focus on the training complexity.
  Finally, we present experimental results, using OpenML and Kaggle datasets, that show that \name{} is able to achieve better generalization loss than competing boosting frameworks, without taking significantly longer to tune.
  
\end{abstract}

% !TEX root = main.tex

\section{Introduction}
\label{sec:introduction}

Boosted ensembles of decision trees are the dominant machine learning (ML) technique today in application domains
where tabular data is abundant (e.g., competitive data science, financial/retail industries). 
While these methods achieve best-in-class generalization, they also expose a large number of hyper-parameters.
The fast training routines offered by modern boosting frameworks allow one to effectively tune these hyper-parameters and are an equally important factor in their success.

The idea of boosting, or building a strong learner from a sequence of weak learners, originated in the early 1990s \cite{schapire1990strength}, \cite{freund1995boosting}.
This discovery led to the widely-popular AdaBoost algorithm~\cite{freund1995decision}, which iteratively trains a sequence of weak learners, whereby the training examples for the next learner are weighted according to the success of the previously-constructed learners. 
An alternative theoretical interpretation of AdaBoost was presented in~\cite{friedman2001greedy}, which showed that the algorithm is equivalent to minimizing an exponential loss function using gradient descent in a functional space. Moreover, the same paper showed that this idea can be applied to arbitrary differentiable loss functions.

The modern explosion of boosting can be attributed primarily to the rise of two software frameworks: XGBoost \cite{KDD16_xgboost} and LightGBM \cite{NIPS17_Ke}. 
Both frameworks leverage the formulation of boosting as a functional gradient descent, to support a wide range of different loss functions, resulting in general-purpose ML solutions that can be applied to a wide range of problems.
Furthermore, these frameworks place a high importance on training performance: employing a range of algorithmic optimizations to reduce complexity (e.g., splitting nodes using histogram summary statistics) as well as system-level optimizations to leverage both many-core CPUs and GPUs.
One additional characteristic of these frameworks is that they use a second-order approximation of the loss function and perform an algorithm akin to Newton's method for optimization.
While this difference with traditional gradient boosting is often glossed over, in practice it is found to significantly improve generalization \cite{sigrist2018gradient}. 

From a theoretical perspective, boosting algorithms are not restricted to any particular class of weak learners. 
At each boosting iteration, a weak learner (from this point forward referred to as a \textit{base hypothesis}) is chosen from some base hypothesis class.
In both of the aforementioned frameworks, this class comprises all binary decision trees up to a fixed maximum depth.
%(or maximum number of leaves in the case of LightGBM).
Moreover, both frameworks are \textit{homogeneous}: the hypothesis class is fixed at each boosting iteration.
Recently~\cite{cortes2014deep, sigrist2019ktboost, cortes2019regularized} have considered \textit{heterogeneous} boosting, in which the hypothesis class may vary across boosting iterations. Promising results indicate that this approach may improve the generalization capability of the resulting ensembles, at the expense of significantly more complex training procedures. 

The goal of our work is to build upon the ideas of \cite{cortes2014deep, sigrist2019ktboost, cortes2019regularized}, and develop a heterogeneous boosting framework with theoretical convergence guarantees, that can achieve better generalization than both XGBoost and LightGBM, without significantly sacrificing performance.

\paragraph{Contributions.}
The contributions of this work can be summarized as follows:

\begin{itemize}[leftmargin=0.25in]
	
\item We propose a Heterogeneous Newton Boosting Machine (HNBM), in which the base hypothesis class at each boosting iteration is selected at random, from a fixed set of subclasses, according to an arbitrary probability mass function $\Phi$.
\item We derive a global linear convergence rate for the proposed HNBM for strongly convex loss functions with Lipschitz-continuous gradients. 
      Our convergence rates agree with existing global rates in the special case when the base hypotheses are fully dense in the prediction space.
\item We describe a particular realization of a HNBM, \name{}, that randomly selects between $K$ different subclasses at each boosting iteration: $(K-1)$ of these subclasses correspond to binary decision trees (BDTs) of different maximum depths and one subclass corresponds to linear regressors with random Fourier features (LRFs).
%that randomly selects between a decision tree or a kernel ridge regressor at each boosting iteration. 
%If the tree hypothesis class is chosen, then a second random selection from the tree hypothesis sub-classes of some given maximum tree depths is performed.
	  We provide details regarding how \name{} is implemented, with a focus on training complexity.
\item We present experiments using OpenML~\cite{OpenML2013} and Kaggle~\cite{kaggle} datasets that demonstrate \name{} generalizes better than competing boosting frameworks, without compromising performance.
\end{itemize}

\subsection{Related Work}
\label{sec:related}
\textbf{Heterogeneous Boosting.}
%A number of existing works have considered boosted ensembles of heterogeneous base learners. 
In \cite{sigrist2019ktboost}, the author proposes a heterogeneous boosting algorithm, KTBoost, that learns both a binary decision tree (BDT) and a kernel ridge regressor (KRR) at each boosting iteration, and selects the one that minimizes the training loss.
It is argued that by combining tree and kernel regressors, such an ensemble is capable of approximating a wider range of functions than trees alone.
While this approach shows promising experimental results, and was a major source of inspiration for our work, the complexity of the training procedure does not scale: one must learn multiple base hypotheses at every iteration. 
%Furthermore, explicitly choosing the base hypothesis that minimizes the training loss may cause additional over-fitting. 
In \cite{cortes2014deep}, the authors derive generalization bounds for heterogeneous ensembles and propose a boosting algorithm, DeepBoost, that chooses the base hypothesis at each boosting iteration by explicitly trying to minimize said bounds. 
The authors acknowledge that exploring the entire hypothesis space is computationally infeasible, and propose a greedy approach that is specific to decision trees of increasing depth. 

%Furthermore, kernel ridge regressors are vastly slower to train than decision trees, unless certain approximations are employed (see Section \ref{sec:implementation}).

\textbf{Randomized Boosting.} 
Stochastic behavior in boosting algorithms has a well-established history \cite{friedman2002stochastic}, and it is common practice today to learn each base hypothesis using a random subset of the features and/or training examples.
Recently, a number of works have introduced additional stochasticity, in particular when selecting the base hypothesis class at each boosting iteration.
In \cite{lu2018randomized}, a randomized gradient boosting machine was proposed that selects, at each boosting iteration, a subset of base hypotheses according to some uniform selection rule. 
The HNBM proposed in our paper can be viewed as a generalization of this approach to include (a) arbitrary non-uniform sampling of the hypothesis space and (b) second-order information. 
A form of non-uniform sampling of the hypothesis space was also considered in \cite{cortes2019regularized}, however second-order information was absent.
 
\textbf{Ordered Boosting.} 
An orthogonal research direction tries to improve the generalization capability of boosting machines by changing the training algorithm to avoid \textit{target leakage}.
CatBoost~\cite{prokhorenkova2018catboost} implements this idea, together with encoding of categorical variables using ordered target statistics, oblivious decision trees, as well as minimal variance example sampling~\cite{NIPS2019_9645}.% a method that decreases the number of examples needed for each boosting iteration. 

\textbf{Deep-learning-based Approaches.} 
\cite{popov2019neural, arik2019tabnet, ke2019deepgbm} have introduced differentiable architectures that are in some sense analogous to boosting machines. 
Rather than using functional gradient descent, these models are trained using end-to-end back-propagation and implemented in automatic differentiation frameworks, e.g., TensorFlow, PyTorch. 
While the experimental results are promising, a major concern with this approach is the comparative training and/or tuning time, typically absent from the papers. 
We compare the methods of our paper with one such approach in Appendix~\ref{sec:node-mixboost} and find that the deep learning-based approach is 1-2 orders of magnitude slower in terms of tuning time.

% !TEX root = main.tex

\section{Heterogeneous Newton Boosting}
\label{sec:algorithm}

In this section we introduce heterogeneous Newton boosting and derive theoretical guarantees on its convergence under certain assumptions.

\subsection{Preliminaries}

We are given a matrix of feature vectors $X\in\mathbb{R}^{n\times d}$ and a vector of training labels $y\in\mathcal{Y}^n$, 
where $n$ is the number of training examples and $d$ is the number of features. 
The $i$-th training example is denoted $x_i^T\in\mathbb{R}^{d}$.
We consider an optimization problem of the form:
\begin{equation}
	\min_{f\in\mathcal{F}} \sum_{i=1}^n l(y_i, f(x_i)),\label{eq:obj}
\end{equation}
where loss function $l:\mathcal{Y}\times\mathbb{R}\rightarrow\mathbb{R^+}$, and $\mathcal{F}$ is a particular class of functions to be defined in the next section. 
We assume that the loss function $l(y,f)$ is twice differentiable with respect to $f$, $l'(y,f)$ and $l''(y,f)$ denote the first and second derivative respectively,
and satisfies the following assumptions:
\begin{assumption}[$\mu$-strongly convex loss] There exists a constant $\mu>0$ such that $\forall y,f_1,f_2$:
	\label{assumption:convex}
	\begin{equation*}
		l(y, f_1) \geq l(y, f_2) + l'(y,f_2)(f_1-f_2) + \frac{\mu}{2}(f_1-f_2)^2 \iff l''(y,f)\geq \mu 
	\end{equation*}
\end{assumption}
\begin{assumption}[$S$-Lipschitz gradients] There exists a constant $S>0$ such that $\forall y,f_1,f_2$:
	\label{assumption:lipschitz}
	\begin{equation*}
	\left|l'(y,f_1) - l'(y,f_2)\right| \leq S|f_1-f_2|\iff l''(y,f) \leq S
	\end{equation*}
\end{assumption}
Examples of loss functions that satisfy the above criteria are the standard least-squares loss: $l(y,f)=\frac{1}{2}(y-f)^2$, as well as L2-regularized logistic loss: $l(y,f)= \log(1 + \exp(-yf))+\frac{\lambda}{2}f^2$ for $\lambda>0$.

\subsection{Heterogeneous Newton Boosting}

We consider a heterogeneous boosting machine in which the base hypothesis at each boosting iteration can be drawn from one of $K$ distinct subclasses.
Let $\mathcal{H}^{(k)}$ denote the $k$-th subclass for $k\in[K]$ which satisfies the following assumption:
\begin{assumption}[Subclass structure] For $k\in[K]$
	\label{assumption:norm}
	\begin{equation*}
		\mathcal{H}^{(k)} = \left\{\sigma b(x) : b(x)\in\mathcal{\bar{H}}^{(k)}, \sigma\in\mathbb{R}\right\},
	\end{equation*}
	where $\mathcal{\bar{H}}^{(k)}$ is a finite class of functions $b:\mathbb{R}^d\rightarrow\mathbb{R}$ that satisfy $\sum_{i=1}^n b(x_i)^2 = 1$.
\end{assumption}
We note that while the subclasses used in practice (e.g., trees) may well be infinite beyond a simple scaling factor, in practice they are finite when represented in floating point arithmetic. 
We now consider the optimization problem (\ref{eq:obj}) over the domain:
\begin{equation}
	\mathcal{F} = \left\{ \sum_{m=1}^{M} \alpha_m b_{m}(x) : \alpha_m\in\mathbb{R},  b_{m}\in \left\{\mathcal{H}^{(1)} \cup \mathcal{H}^{(2)} \ldots \cup \mathcal{H}^{(K)} \right\}\right\}.\label{eq:domain}
\end{equation}

Our proposed method for solving this optimization problem is presented in full in Algorithm \ref{alg:hnbm}. 
At each boosting iteration, we randomly sample one of the $K$ subclasses according to a given probability mass function (PMF) $\Phi$.
The probability that the $k$-th subclass is selected is denoted $\phi_k$.
Let $u_m\in[K]$ denote the sampled subclass index at boosting iteration $m$.
The base hypothesis to insert at the $m$-th boosting iteration is determined as follows:
\begin{equation}
	b_m = \argmin_{b\in \mathcal{H}^{(u_m)}}\left[\sum_{i=1}^n l(y_i, f^{m-1}(x_i) + b(x_i))\right] 
	%&\approx \argmin_{j\in I(u_m)}\left[\sum_{i=1}^n l(y_i, f(x_i)) + g_i b_j(x_i) + \frac{h_i}{2}b_j(x_i)^2\right] \nonumber \\
	\approx \argmin_{b\in \mathcal{H}^{(u_m)}}\left[\sum_{i=1}^n h_i\left(-g_i/h_i - b(x_i)\right)^2\right],\label{eq:regression}
\end{equation}
where the approximation is obtained by taking the second-order Taylor expansion of $l(y_i, f^{m-1}(x_i) + b(x_i))$ around $f^{m-1}(x_i)$ and 
the expansion coefficients are given by $g_i = l'(y_i, f^{m-1}(x_i))$ and $h_i = l''(y_i,f^{m-1}(x_i))$.
In practice, an L2-regularization penalty, specific to the structure of the subclass, may also be applied to \eqref{eq:regression}.
It should be noted that (\ref{eq:regression}) corresponds to a standard sample-weighted least-squares minimization which, depending on the choice of subclasses,
enables one to reuse a plethora of existing learning algorithms and implementations 
\footnote{The supplemental material contains exemplary code for Algorithm \ref{alg:hnbm} that uses generic scikit-learn regressors.}.
Intuitively, the algorithm chooses the hypothesis from the randomly selected subclass $\mathcal{H}^{(u_m)}$ that is closest (in a Euclidean sense) to the Newton descent direction, 
and dimensions with larger curvature are weighted accordingly.
To ensure global convergence, the model is updated by applying a learning rate $\epsilon>0$:
\begin{equation*}
	f^{m}(x) = f^{m-1}(x) + \epsilon b_{m}(x).
\end{equation*}
In practice, $\epsilon$ is normally treated as a hyper-parameter and tuned using cross-validation, although some theoretical insight on how it should be set to ensure convergence is provided later in the section.

\begin{algorithm}[t]
\begin{algorithmic}[1]
\State \textbf{initialize: } $f^0(x) = 0$
\For {$m=1,\ldots,M$}
	\State Compute vectors $\left[g_i\right]_{i=1,\ldots,n}$ and $\left[h_i\right]_{i=1,\ldots,n}$
	\State Sample subclass index $u_m\in\{1,2,\ldots,K\}$ according to probability mass function $\Phi$
	\State Fit base hypothesis: $b_m = \argmin_{b \in \mathcal{H}^{(u_m)}} \sum_{i=1}^n h_i\left(-g_i/h_i - b(x_i)\right)^2$ 
	\State Update model: $f^{m}(x) = f^{m-1}(x) + \epsilon b_{m}(x)$
\EndFor
\State \textbf{output: } $f^M(x)$
\end{algorithmic}
\caption{Heterogeneous Newton Boosting Machine}
\label{alg:hnbm}
\end{algorithm}

\subsection{Reformulation as Coordinate Descent}

While boosting machines are typically implemented as formulated above, they are somewhat easier to analyze theoretically when viewed instead as a coordinate descent in a very high dimensional space \cite{lu2018randomized, cortes2019regularized}. 
Let $\mathcal{\bar{H}}=\mathcal{\bar{H}}^{(1)}\cup\mathcal{\bar{H}}^{(2)}\cup\ldots\cup\mathcal{\bar{H}}^{(K)}$ denote the union of the finite, normalized subclasses defined in Assumption \ref{assumption:norm}.
Furthermore, let $b_j\in\mathcal{\bar{H}}$ denote an enumeration of the hypotheses for $j\in[|\mathcal{\bar{H}}|]$ and $I(k)=\{j : b_j\in\mathcal{\bar{H}}^{(k)}\}$ denote the set of all indices 
corresponding to normalized hypotheses belonging to the $k$-th subclass. 
Let $B\in\mathbb{R}^{n\times|\mathcal{\bar{H}}|}$ be a matrix with entries given by $B_{i,j} = b_j(x_i)$.
Then we can reformulate our optimization problem (\ref{eq:obj}) over domain \eqref{eq:domain} as follows:
\begin{equation*}
	\min_{\beta\in\mathbb{R}^{|\mathcal{\bar{H}}|}}L(\beta) \equiv \min_{\beta\in\mathbb{R}^{|\mathcal{\bar{H}}|}}\sum_{i=1}^n l(y_i, B_i \beta),
\end{equation*}
where $B_i$ denotes the $i$-th row of $B$. In this reformulation, the model at iteration $m$ is given by:
\begin{equation}
	\label{eq:update}
	\beta^{m} = \beta^{m-1} + \epsilon\sigma^*_{j_m}e_{j_m},
\end{equation}
where $e_{j}$ denotes a vector with value $1$ in the $j$-th coordinate and $0$ otherwise, $\sigma_j$ is the descent magnitude, and $j_m$ is the descent coordinate. 
For a given $j$, the magnitude $\sigma^*_{j}$ is given by:
\begin{equation}
	\sigma^*_j = \min_{\sigma}\sum_{i=1}^n h_i\left(-g_i/h_i - \sigma B_{i,j}\right)^2 = -\nabla^2_{j}L(\beta^{m-1})^{-1} \nabla_{j}L(\beta^{m-1}),
	\label{eq:update_magnitude}
\end{equation}
and the descent coordinate is given by:
\begin{equation}
	j_m = \argmin_{j\in I(u_m)}\left[\sum_{i=1}^n h_i\left(-g_i/h_i - \sigma^*_j B_{i,j}\right)^2\right] = \argmax_{j\in I(u_m)}\left[\left|\nabla^2_{j}L(\beta^{m-1})^{-1/2} \nabla_j L(\beta^{m-1})\right|\right].
	\label{eq:update_coord}
\end{equation}
Further details regarding this reformulation are provided in Appendix \ref{app:coord}.

\subsection{Theoretical Guarantees}

In order to establish theoretical guarantees for Algorithm \ref{alg:hnbm}, we adapt the theoretical framework developed in \cite{lu2018randomized} to our setting with non-uniform sampling of the subclasses as well as the second-order information. 
In particular, we derive a convergence rate that depends on the following quantity:
\begin{definition}[Minimum cosine angle] The minimum cosine angle $0\leq\Theta\leq1$ is given by:
	\label{defn:mca}
	 \begin{equation}
		\Theta = \min_{c\in Range(B)} \left\| \left[\cos(B_{.j}, c)\right]_{j=1\ldots,|\mathcal{\bar{H}}|}\right\|_{\Phi},
	\end{equation}
	where $B_{.j}$ denotes the $j$-th column of the matrix $B$ and $\left\|x\right\|_{\Phi} = \sum_{k=1}^K \phi_k \max_{j\in I(k)}|x_j|$.
\end{definition}
The minimum cosine angle measures the expected \textit{density} of base hypotheses in the prediction space. 
A value close to $1$ indicates that the Newton direction can be closely fitted to one of the base hypotheses, and a value close to $0$ the opposite. 

In order to prove global convergence of Algorithm \ref{alg:hnbm}, we will need the following technical lemma:
\begin{lemma}
	Let $\mathbb{E}_m[.]$ denote expectation over the subclass selection at the $m$-th boosting iteration
	and let $\Gamma(\beta) = \left[\nabla^2_{j}L(\beta)^{-1/2} \nabla_j L(\beta)\right]_{j=1\ldots,|\mathcal{\bar{H}}|}$
	then the following inequality holds:
	\label{lemma:exp}
	\begin{equation}
	\mathbb{E}_m\left[\Gamma_{j_m}(\beta^{m-1})^2\right] \geq \left\|\Gamma(\beta^{m-1})\right\|_{\Phi}^2.
	\end{equation}
\end{lemma}
The proof is provided in Appendix \ref{app:lemma}.
With this result in hand, one can prove the following global linear convergence rate for Algorithm \ref{alg:hnbm}:
\begin{theorem}
	\label{thm:converge}
	Given learning rate $\epsilon=\mu/S$ then:
	\begin{equation}
	\mathbb{E}\left[L(\beta^M)-L(\beta^*)\right] \leq \left(1-\frac{\mu^2}{S^2}\Theta^2\right)^M\left(L(\beta^0)-L(\beta^*)\right),
	\end{equation} 
	where the expectation is taken over the subclass selection at all boosting iterations.
\end{theorem}
\begin{proof}
	For clarity, we provide only a sketch of the proof, with the full proof in Appendix \ref{app:theorem}.
	Starting from the coordinate update rule (\ref{eq:update}) we have:
	\begin{equation}
		L(\beta^{m}) = L\left(\beta^{m-1} - \epsilon \sigma_{j_m}^* e_{j_m}\right)
		\leq L(\beta^{m-1}) - \frac{\mu}{2S} \Gamma_{j_m}(\beta^{m-1})^2,
	\end{equation}
	where the inequality is obtained by applying the mean value theorem and using Assumptions \ref{assumption:convex} and \ref{assumption:lipschitz}.
	Now taking expectation over the $m$-th boosting iteration and applying Lemma \ref{lemma:exp} we have:
	\begin{align}
		\mathbb{E}_m\left[L(\beta^{m})\right] \leq L(\beta^{m-1}) - \frac{\mu}{2S} \left\|\Gamma(\beta^{m-1})\right\|_{\Phi}^2
		\leq  L(\beta^{m-1}) - \frac{\mu}{2S^2} \left\|\nabla L(\beta^{m-1})\right\|_{\Phi}^2,
		\label{eq:upper}
	\end{align}
	where the second inequality is due to Assumptions \ref{assumption:lipschitz} and \ref{assumption:norm}.
	We then leverage Assumption \ref{assumption:convex} together with Proposition 4.4 and Proposition 4.5 from \cite{lu2018randomized} to obtain the following lower bound:
	\begin{equation}
		\left\|\nabla L(\beta^{m-1})\right\|_{\Phi}^2 \geq 2\mu\Theta^2\left(L(\beta^{m-1})-L(\beta^*)\right)
		\label{eq:lower}
	\end{equation}
	Then, by subtracting $L(\beta^*)$ from both sides of (\ref{eq:upper}), plugging in (\ref{eq:lower}), and following a telescopic argument, the desired result is obtained.
\end{proof}
We note that in the case where $\Theta=1$ (i.e., implying there always exists a base hypothesis that perfectly fits the Newton descent vector) the rate above is equivalent to that derived in \cite{karimireddy2018global} for Newton's method under the same assumptions.

% !TEX root = main.tex

\section{\name: A Heterogeneous Newton Boosting Machine}
\label{sec:implementation}

In this section, we describe \name{}, a realization of a HNBM that admits a low-complexity implementation. 
\name{} is implemented in C++ and uses OpenMP for parallelization and Eigen \cite{eigen} for linear algebra.
The algorithm is exposed to the user via a sklearn-compatible Python API.

\subsection{Base Hypothesis Subclasses}
At each boosting iteration, \name{} chooses the subclass of base hypotheses to comprise binary decision trees (BDTs) with probability $p_t$ or linear regressors with random Fourier features (LRFs) with probability $(1-p_t)$.
Furthermore, if BDTs are selected, the maximum depth of the trees in the subclass is chosen uniformly at random between $D_{min}$ and $D_{max}$, resulting in $K=N_D+1$ unique choices for the subclass at each iteration, where $N_D=D_{max}-D_{min}+1$.
The corresponding PMF is given by: $\Phi=[\frac{p_t}{N_D},\ldots,\frac{p_t}{N_D}, 1-p_t]$.
Note that the PMF $\Phi$ is fully parameterized by $p_t$, $D_{min}$ and $D_{max}$.
A full list of hyper-parameters is provided in Appendix \ref{sec:snapboost-hp}.

\subsection{Binary Decision Trees}
In a regression tree,  each node represents a test on a feature,  each branch the outcome of the test and each leaf node a continuous value. 
The tree is trained using all or a subsample of the examples and features in the train set, where the example/feature sampling ratios ($r_{n}$ and $r_{d}$) are hyper-parameters.
In order to identify the best split at each node, one must identify the feature and feature value which, if split by, will optimize \eqref{eq:regression}. 
%The value of a leaf is computed as the average of the labels of the train examples that end up during training in that leaf. 
The tree-building implementation in \name is defined in three steps as follows. 
Steps 1 and 2 are performed only \emph{once} for all boosting iterations, whereas Step 3 is performed on each node, for each boosting iteration at which a BDT is chosen.

\textbf{Step 1.} 
We sort the train set for each feature~\cite{guillame2018arxiv,mehta1996sliq,shafer96vldb}. 
This step reduces the complexity of finding the best split at each node, which is a critical training performance bottleneck. 
However, it also introduces a one-off overhead:  the sort time, which has a complexity of $O(dn\log(n))$.

\textbf{Step 2.} 
We build a compressed representation of the input dataset to further reduce the complexity of finding the best split at each node. 
We use the sorted dataset from Step 1 to build a histogram~\cite{KDD16_xgboost,Zhang2017GPUaccelerationFL} for each feature.
The number of histogram bins $h$ can be at most 256 and thus typically $h\ll n$.
For each feature, its histogram bin edges are constructed before boosting begins, by iterating over the feature values and following a greedy strategy to balance the number of examples per bin. 
The complexity of building this histogram is $O(dn)$. 
Each histogram bin also includes statistics necessary to accelerate the computation of the optimal splits.
While the bin edges remain fixed across boosting iterations, these statistics are continually recomputed during tree-building.

\textbf{Step 3.} The actual construction of the tree is performed using a depth-first-search algorithm~\cite{DBLP:books/daglib/0023376}. %After a node has been split, the tree-building algorithm recursively traverses the tree through the left-child. Once a leaf node has been reached, it traverse up one level and recursively explores the right-child. 
%Each new node is placed in a last-in-first-out stack. 
For each node, two steps are performed: a) finding the best split, and b) initializing the node children. The complexity of step a) is $O(dr_{d}h)$: instead of iterating through the feature values of each example, we iterate over the histogram bin edges. In step b), we first assign the node examples to the children, an operation of complexity $O(n_{node})$, where $n_{node}$ is the number of examples in the node being split. Then, we update the bin statistics, a step of complexity $O(n_{node} dr_{d})$. 
Assuming a complete tree of depth $D$, the overall complexity of step 3 is $O(2^D dr_{d}h + dr_{d}nr_{n}D)$ for each boosting iteration at which a BDT is chosen. 

\subsection{Linear Regressors with Random Fourier Features}

We use the method proposed in \cite{10.5555/2981562.2981710} to learn a linear regressor on the feature space induced by a random projection matrix, designed to approximate a given kernel function.
The process is two-fold:

\textbf{Step 1.} 
First, we map each example $x \in \mathbb{R}^d$ in the train set $X$ to a low-dimensional Euclidean inner product space using a randomized feature map, $z$, that uniformly approximates the Gaussian radial basis kernel $\mathcal{K}(x,x') = \exp(-\gamma ||x-x'||^2)$. 
The feature map  $z:\mathbb{R}^d \rightarrow \mathbb{R}^c$ is defined as $z(x) = \sqrt{2/c} [\cos(\xi^T_1 x+\tau_1), ... \cos(\xi^T_c x+\tau_c)]^T$, where the weights $\xi_i$ are i.i.d samples from the Fourier transform of the Gaussian kernel and the offsets $\tau_i$ uniformly drawn from $[0,2 \pi]$. 
The complexity of projecting the train set onto the new feature space is essentially given by the multiplication of the feature matrix $X\in\mathbb{R}^{n\times d}$ with the weights matrix $\xi\in\mathbb{R}^{d\times c}$, an operation of complexity $O(n d c)$. 
Similarly to the tree histograms, the randomized weights and offsets are generated only \emph{once}, for all boosting iterations. 
The dimensionality of the projected space, $c$, is a hyper-parameter and is typically chosen as $c<100$.

\textbf{Step 2.} 
Using the projection of the train set as input $X'\in\mathbb{R}^{n\times c}$, we solve the sample-weighted least-squares problem defined in \eqref{eq:regression}, adding L2-regularization as follows: 
$\sum_{i=1}^n h_i \left(y'_i - w^T x'_i\right)^2 + \alpha \left\Vert w \right\Vert^2$, where $y'_i = -g_i/h_i$ are the regression targets.
Given the low dimensionality, $c$, of the new feature space, we solve the least-squares problem by computing its closed-form solution $(X'^T X' + \alpha I)^{-1} X' y'$.
The complexity of computing this solution is dominated by the complexity of the $X'^T X'$ operation $O(n c^2)$ or the complexity of the inversion operation $O(c^3)$. 
As $c\ll n$, the complexity of this step is $O(n c^2)$, for each boosting iteration at which a LRF is chosen.

Which step dominates the overall complexity of \name{} strongly depends on the range of tree depths, $D_{min}$ and $D_{max}$, the dimensionality of the projected space, $c$, as well as the PMF, $\Phi$, that controls the mixture.
When performance is paramount, one may enforce constraints on $\Phi$ (e.g., $p_t \geq 0.9$) to explicitly control the complexity by favoring BDTs over LRFs or vice-versa.

% !TEX root = main.tex

\section{Experimental Results}
\label{sec:experiments}

In this section, we evaluate the performance of \name against widely-used boosting frameworks.

\textbf{Hardware and Software.} 
The results in this section were obtained using a multi-socket server with two 20-core Intel(R) Xeon(R) Gold 6230 CPUs @2.10GHz, 256 GiB RAM, running Ubuntu 18.04. 
We used XGBoost v1.1.0, LightGBM v2.3.1, CatBoost v.0.23.2 and KTBoost v0.1.13.

\textbf{Hyper-parameter Tuning.} 
All boosting frameworks are tuned using the successive halving (SH) method \cite{pmlr-v51-jamieson16}. 
Our SH implementation is massively parallel and leverages process-level parallelism, as well as multi-threading within the training routines themselves. 
Details regarding the SH implementation and the hyper-parameter ranges can be found in Appendix \ref{app:sh} and Appendix \ref{app:hp-space} respectively.

\begin{table}[]
\centering
\caption{Average test loss using 3x3 nested cross-validation for the OpenML datasets.}
\label{tab:openml}
\resizebox{\textwidth}{!}{%
\begin{tabular}{llcclll}
\hline
\textbf{ID} & \textbf{Name}    & \multicolumn{1}{l}{\textbf{Examples}} & \multicolumn{1}{l}{\textbf{Features}} & \textbf{XGBoost}        & \textbf{LightGBM}       & \textbf{\name{}}                \\ \hline
4154        & CreditCardSubset & 14240                                 & 30                                    & 3.9990e-01              & 4.3415e-01              & \textbf{3.8732e-01}              \\
1471        & eeg-eye-state    & 14980                                 & 14                                    & 1.3482e-01              & 1.4184e-01              & \textbf{1.2966e-01}              \\
4534        & PhishingWebsites & 11055                                 & 30                                    & 7.3332e-02              & \textbf{7.1488e-02}     & 7.2481e-02                       \\
310         & mammography      & 11183                                 & 6                                     & 2.6594e-01              & 2.7083e-01              & \textbf{2.6437e-01}              \\
734         & ailerons         & 13750                                 & 40                                    & 2.5902e-01              & \textbf{2.5667e-01}     & 2.5808e-01                       \\
722         & pol              & 15000                                 & 48                                    & 3.6318e-02              & 3.5973e-02              & \textbf{3.4507e-02}              \\
1046        & mozilla4         & 15545                                 & 5                                     & 1.6973e-01              & 1.6623e-01              & \textbf{1.6300e-01}              \\
1019        & pendigits        & 10992                                 & 16                                    & 1.8151e-02              & 1.9637e-02              & \textbf{1.8057e-02}              \\
959         & nursery          & 12960                                 & 8                                     & 2.3469e-04              & 2.3107e-07              & \textbf{7.0620e-08}              \\
977         & letter           & 20000                                 & 16                                    & 3.0621e-02              & 2.9400e-02              & \textbf{2.6005e-02}              \\ \hline
\multicolumn{4}{r}{Average Rank:}                                                                              & \multicolumn{1}{c}{2.6} & \multicolumn{1}{c}{2.2} & \multicolumn{1}{c}{\textbf{1.2}} \\ \hline
\end{tabular}%
}
\end{table}

\subsection{OpenML Benchmark}

We compare XGBoost, LightGBM and \name{} across 10 binary classification datasets sourced from the OpenML platform \cite{OpenML2013}. 
Details regarding the characteristics of the datasets as well as their corresponding preprocessing steps are presented in Appendix \ref{app:datasets}.
Since the datasets are relatively small (between 10k and 20k examples), 3x3 nested stratified cross-validation was used to perform hyper-parameter tuning and to obtain a reliable estimate of the generalization loss.
For each of the 3 outer folds, we perform tuning using cross-validated SH over the 3 inner folds.
Some of the datasets also exhibit class imbalance, thus a sample-weighted logistic loss is used as the training, validation and test metric. 
The test losses (averaged over the 3 outer folds) are presented in Table \ref{tab:openml}.
We observe that XGBoost does not win on any of the 10 datasets (average rank 2.6), LightGBM wins on 2 (average rank 2.2), whereas \name{} wins on 8/10 of the datasets (average rank 1.2).

\textbf{Statistical Significance. }
When comparing a number of ML algorithms across a large collection of datasets, rather than applying parametric statistical tests (such as Student's t-test) on a per-dataset basis, it is preferable to perform non-parametric tests across the collection \cite{demvsar2006statistical}. 
Firstly, we apply the Iman and Davenport's correction of the Friedman omnibus test \cite{iman1980approximations} to verify differences exist within the family of 3 algorithms ($p < 0.002$).
Secondly, we perform pairwise testing using the Wilcoxon signed-rank test \cite{JMLR:v17:benavoli16a} (correcting for multiple hypotheses via Li's procedure \cite{li2008two}) to verify differences exist between the algorithms themselves. 
We find that the null hypothesis can be safely rejected when comparing \name{} with XGBoost ($p<0.004$) and LightGBM ($p<0.02$). 
However, when comparing XGBoost and LightGBM, the null hypothesis cannot be rejected ($p>0.36$).

\subsection{Kaggle Benchmark}

In order to evaluate the generalization capability and performance of \name{} on more realistic use-cases we use 3 datasets from the Kaggle platform. 
Details of the datasets, as well as the preprocessing pipeline that was used are provided in Appendix \ref{app:datasets}.
Since these datasets are relatively large, we perform a single train/validation/test split.
Hyper-parameter tuning (via SH) is performed using the train and validation sets. 
Once the tuned set of hyper-parameters is obtained, we re-train using the combined train and validation set and evaluate on the test set.
The re-training is repeated 10 times using different random seeds, in order to quantify the role of stochastic effects on the test loss. 

In Figure \ref{fig:kaggle_loss}, we compare XGBoost (XGB), LightGBM (LGB), CatBoost (CAT), KTBoost (KT) and \name{} in terms of test loss
\footnote{Experiments that took longer than 8 hours were killed and do not appear in the plots.}.
To quantify the effect of including LRFs in the ensemble, we present results for \name{} restricted only to BDTs (SB-T) as well as the unrestricted version (SB). 
We observe that in all 3 datasets, SB generalizes better than the frameworks that use only BDTs. 
In Figure \ref{fig:kaggle_loss_cc}, we observe that SB achieves a similar test loss to KT, that uses BDTs and KRRs. 
Next, in Figure \ref{fig:kaggle_time}, we compare the frameworks in terms of experimental time.
We observe that, while the SB time is comparable to that of XGB and LGB, both CAT and KT are significantly slower.
This behaviour is expected for KT since it (a) learns two base hypotheses at every iteration and (b) is implemented using sklearn components.
For CAT, ordered boosting is known to introduce overheads when using a small number of examples\footnote{https://github.com/catboost/catboost/issues/505}, 
which is always the case in the early stages of SH.

\begin{figure*}
  \centering

  \subfigure[Credit Card Fraud \cite{creditcard}] {
  	\includegraphics[width=0.29\columnwidth]{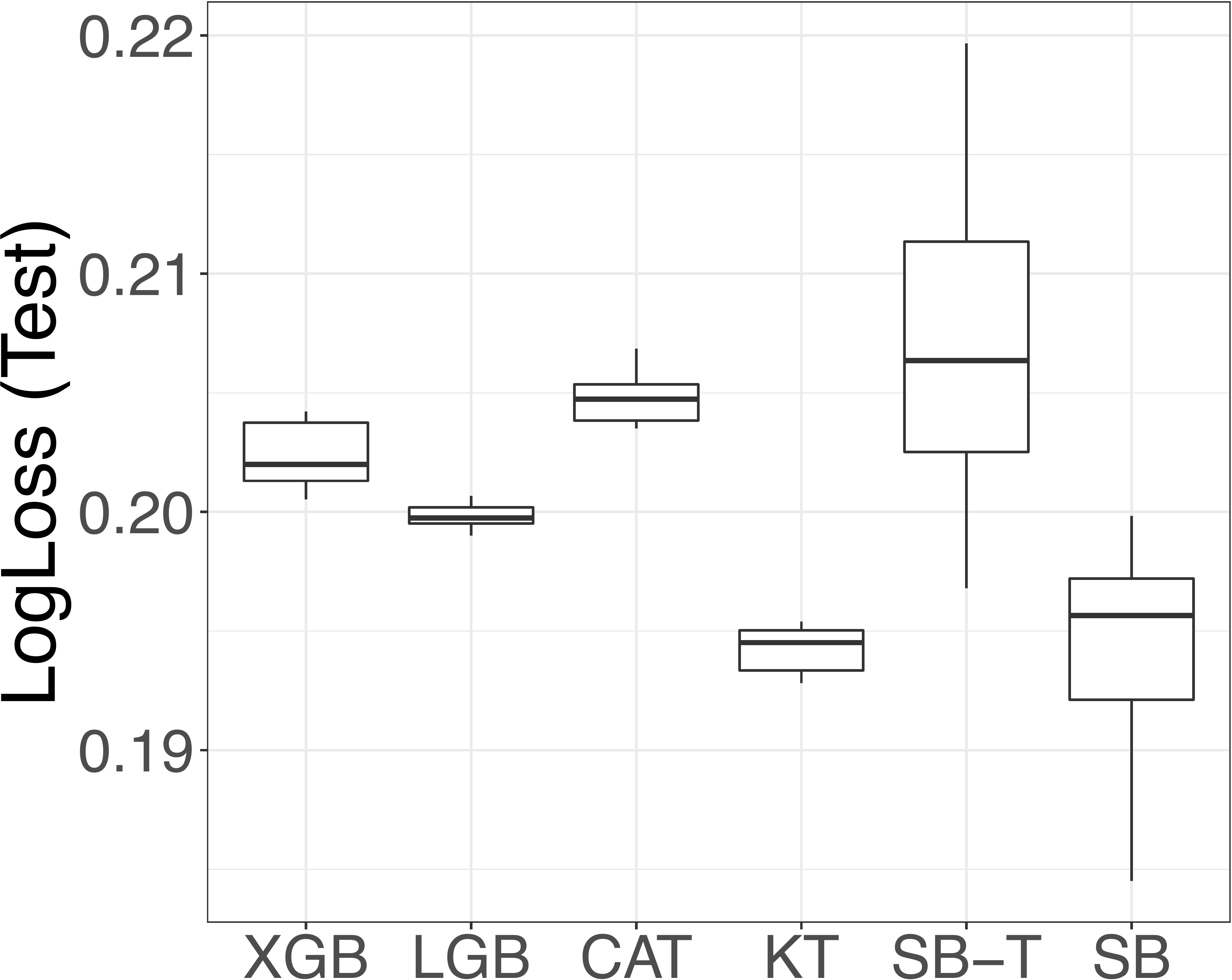}
	\label{fig:kaggle_loss_cc}
  }
  \subfigure[Rossmann Store Sales \cite{rossmann}] {
  	\includegraphics[width=0.29\columnwidth]{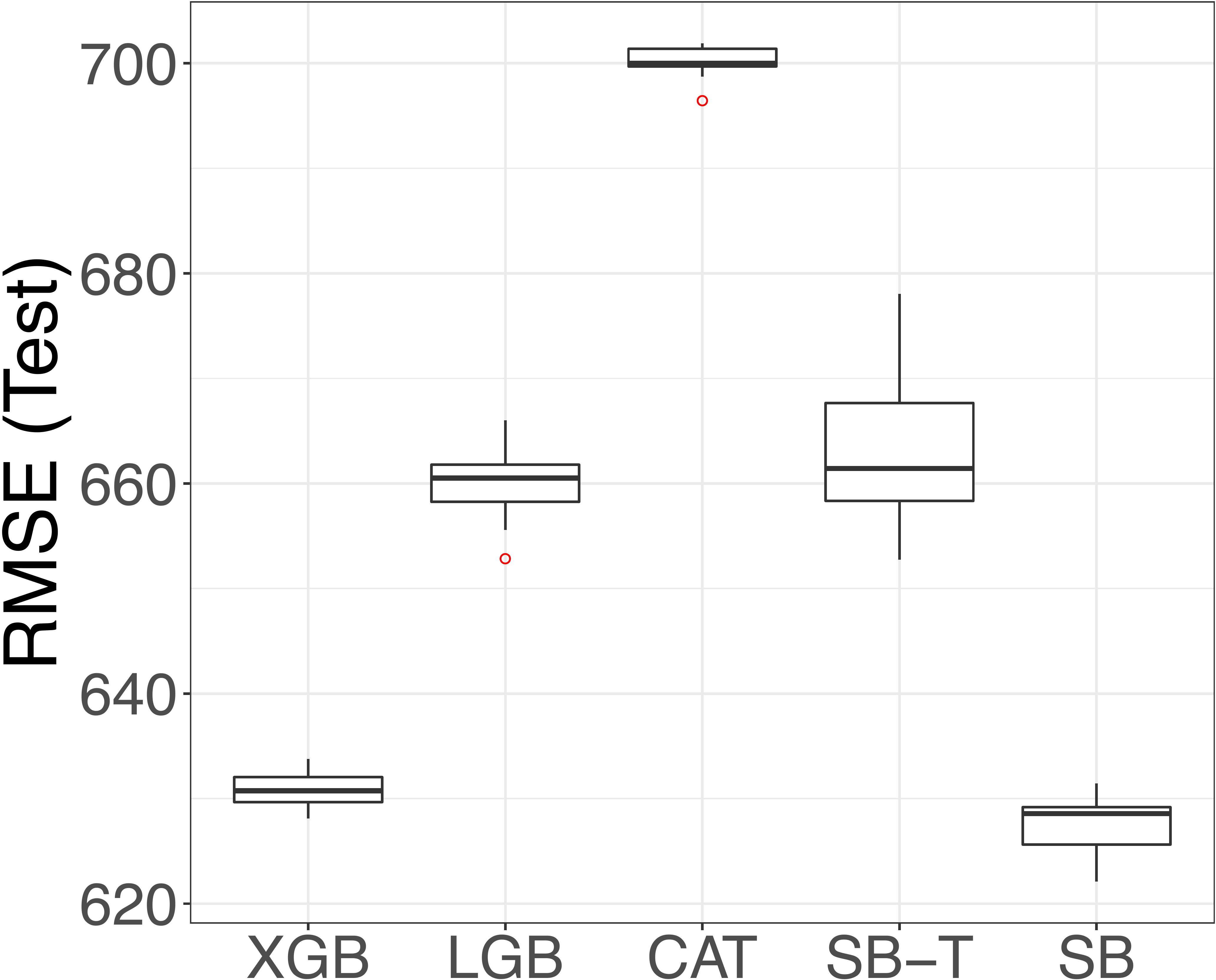}
  	\label{fig:kagge_loss_rm}
  }
  \subfigure[Mercari Price Suggestion \cite{priceprediction}] {
  	\includegraphics[width=0.29\columnwidth]{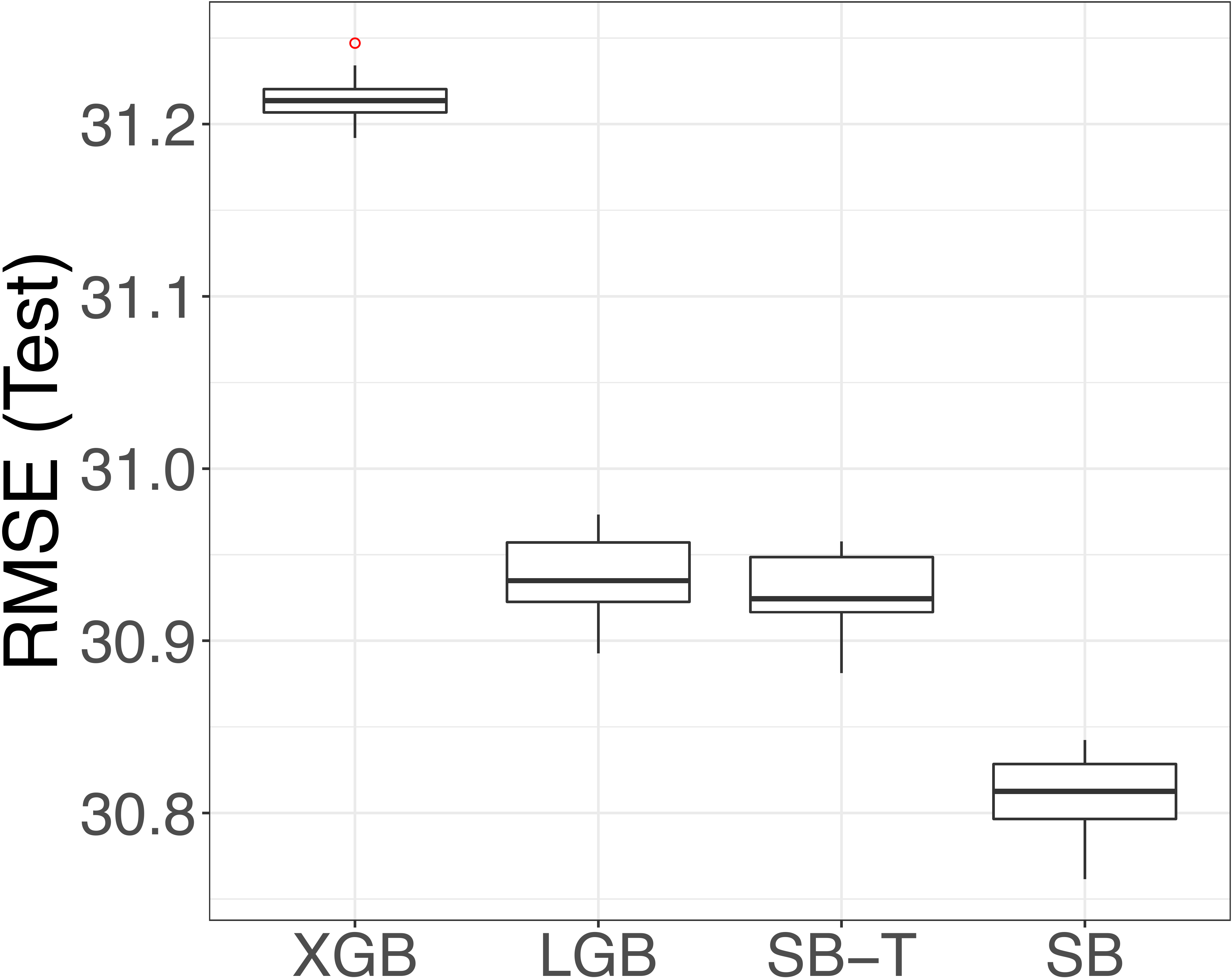}
	\label{fig:kaggle_loss_pp}
  }
  \caption{Test loss for the different boosting frameworks (10 repetitions with different random seeds).}
  \label{fig:kaggle_loss}
\end{figure*}

\begin{figure*}
  \centering

  \subfigure[Credit Card Fraud \cite{creditcard}] {
  	\includegraphics[width=0.29\columnwidth]{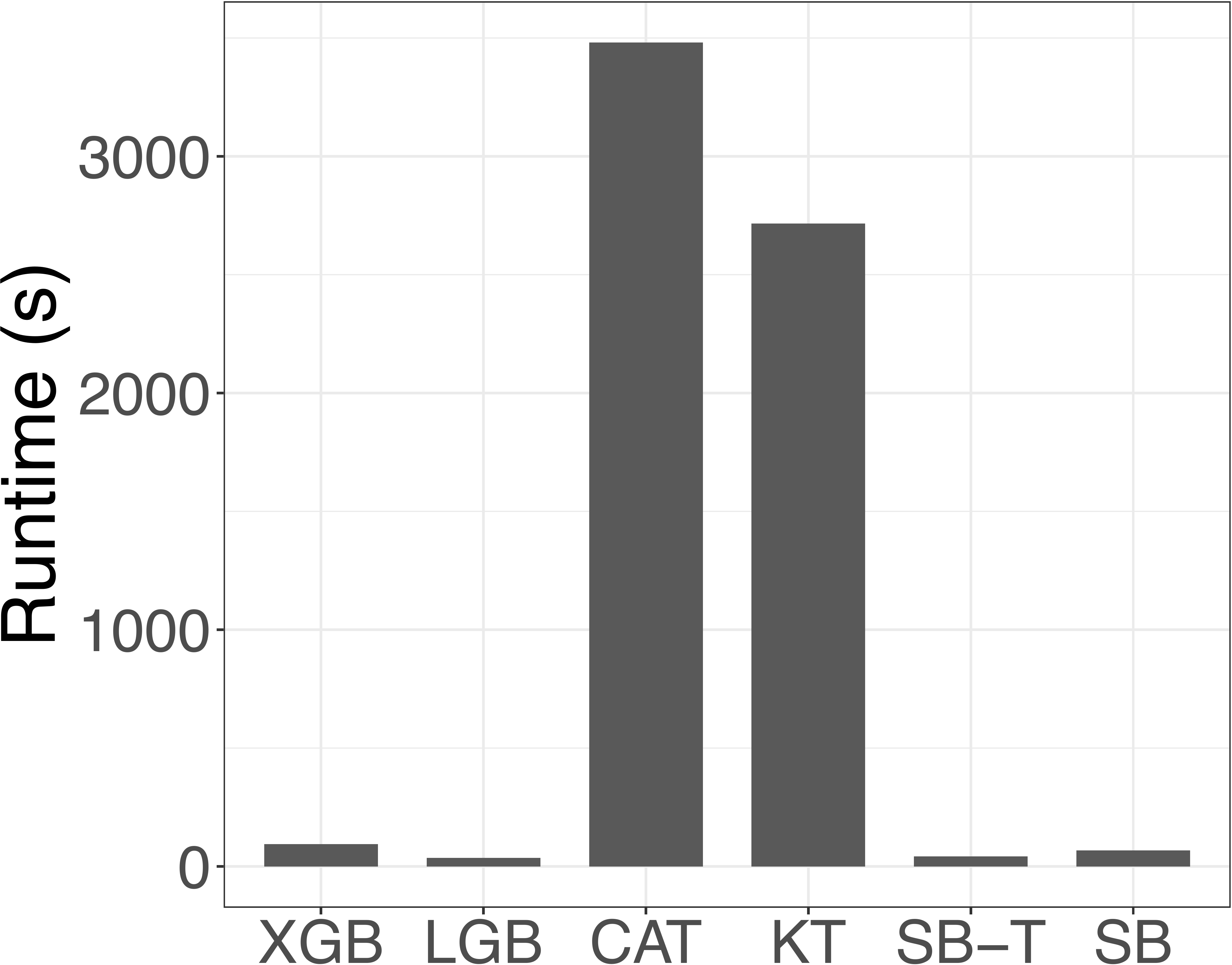}
	\label{fig:kaggle_time_cc} 
  }
  \subfigure[Rossmann Store Sales \cite{rossmann}] {
  	\includegraphics[width=0.29\columnwidth]{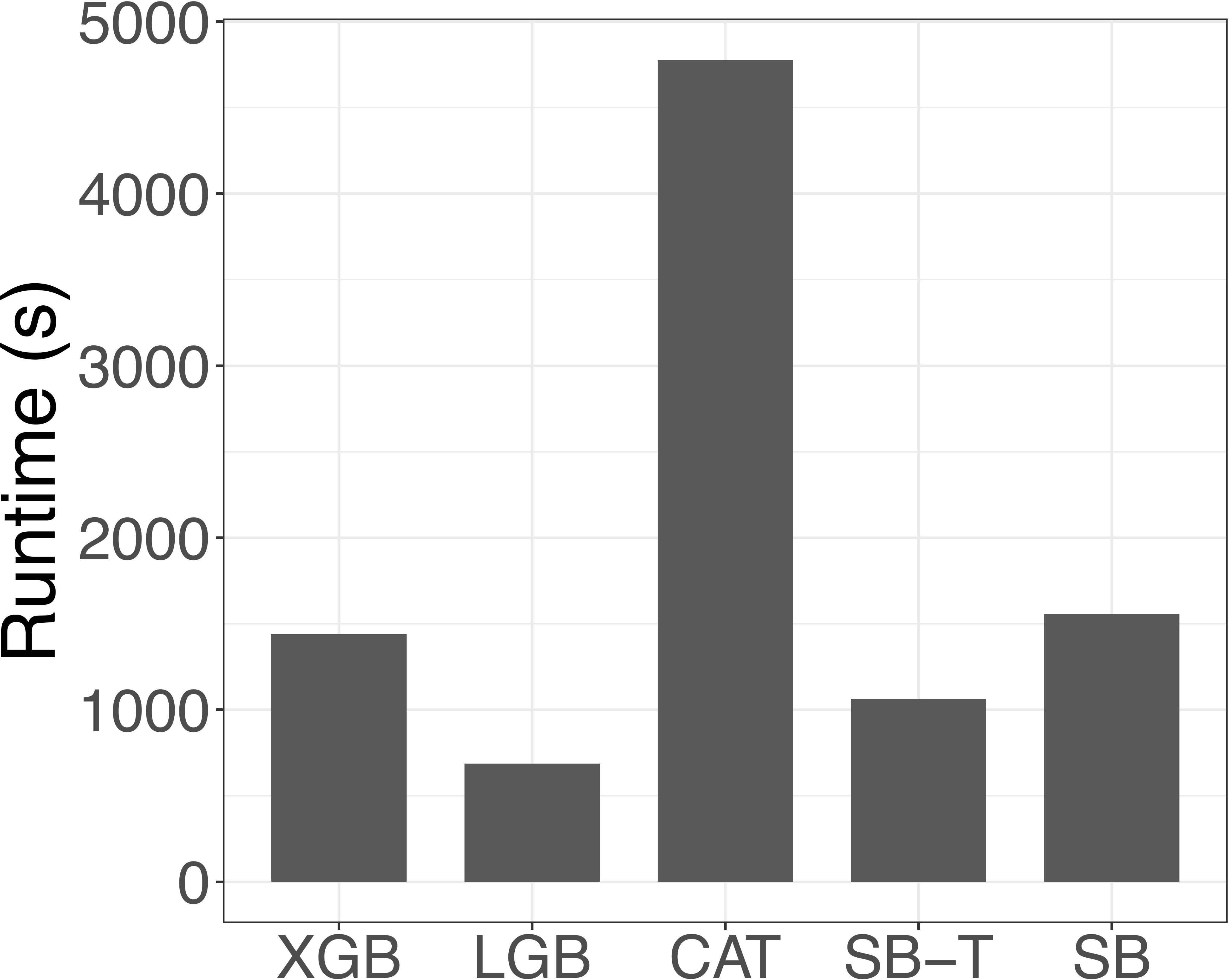}
	\label{fig:kaggle_time_rm}
  }
  \subfigure[Mercari Price Suggestion \cite{priceprediction}] {
  	\includegraphics[width=0.29\columnwidth]{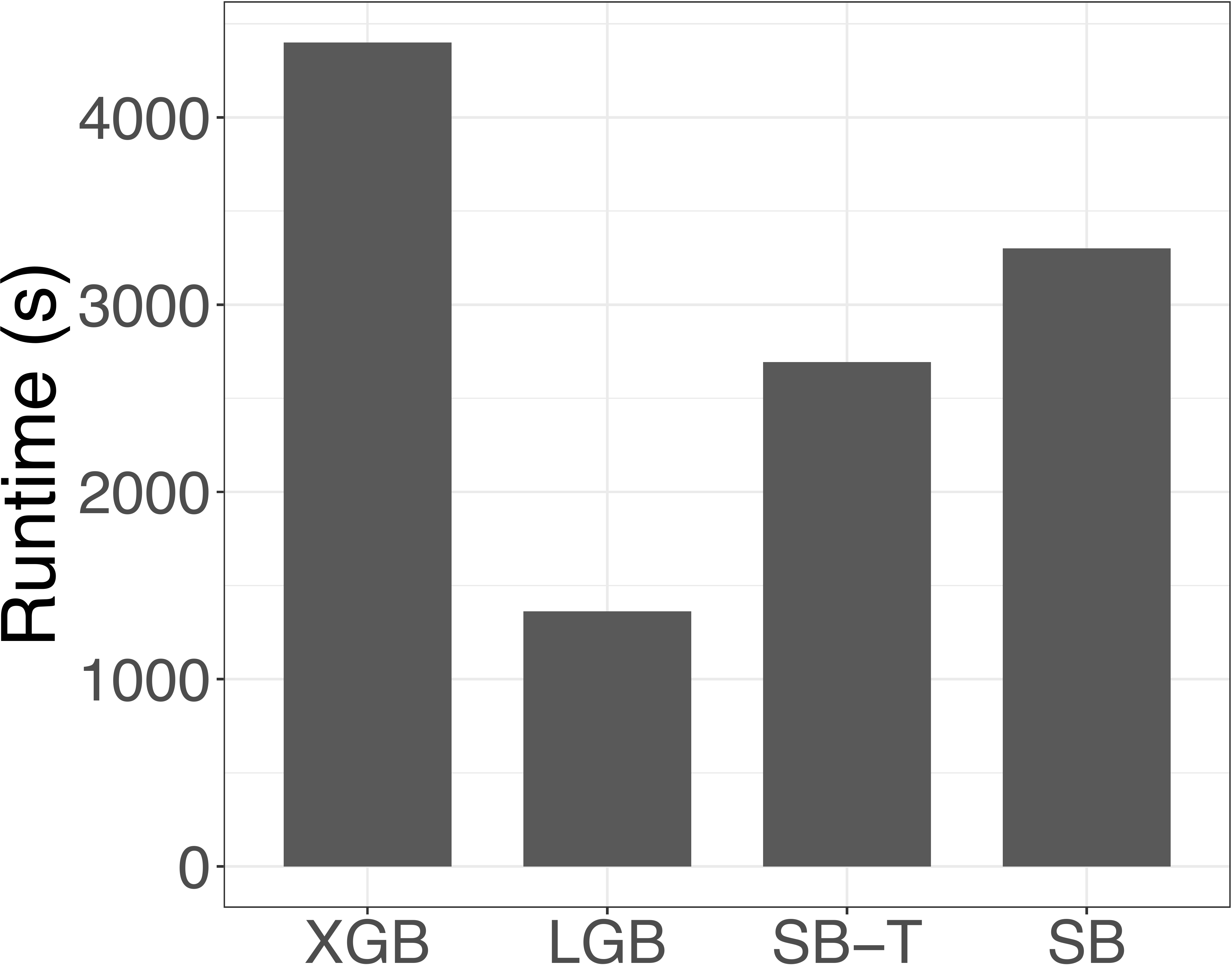}
	\label{fig:kaggle_time_pp}
  }
  \caption{End-to-end experiment time (tuning, re-training and evaluation) for all frameworks.}
  \label{fig:kaggle_time}
\end{figure*}

% !TEX root = main.tex

\section{Conclusion}

In this paper, we have presented a Heterogeneous Newton Boosting Machine (HNBM), with theoretical convergence guarantees, that selects the base hypothesis class stochastically at each boosting iteration.
Furthermore, we have described an implementation of HNBM, \name{}, that learns heterogeneous ensembles of BDTs and LRFs.
Experimental results on 13 datasets indicate that \name{} provides state-of-the-art generalization, without sacrificing performance.
As a next step, we plan to further enhance the performance of \name by taking advantage of GPUs.

\newpage 

% !TEX root = main.tex
\section*{Broader Impact}

Boosting machines are generally considered most effective in application domains involving large amounts of tabular data. 
We have directly encountered use-cases in the retail, financial and insurance industries, and there are likely to be many others.
Such examples include: credit scoring, residential mortgage appraisal, fraud detection and client risk profiling.

The tables used to train these models may contain sensitive personal information such as gender, ethnicity, health, religion or financial status.
It is therefore critical that the algorithms used do not \textit{leak} such information. 
In particular, an adversary should not be able to exploit a trained model to discover sensitive information about an individual or group. 
While we do not address these concerns in this paper, efforts are ongoing in the research community to develop privacy-preserving boosting machines \cite{li2019privacy}. 

Given the application domains where boosting machines are currently deployed, another important issue is fairness. 
Formally, we would like that certain statistics regarding the decisions produced by the trained model are consistent across individuals or groups of individuals. 
This definition imposes new constraints, which our training algorithms must be modified to satisfy.
While this problem has received a significant amount of attention from the community in general, only a few works have looked at designing boosting machines that satisfy fairness constraints \cite{fish2015fair, grari2019fair}.
Given the widespread use of boosting machines in production systems, this is a topic worthy of future investigation. 

{
\small
\bibliographystyle{plain}
\bibliography{bibliography}
}

\newpage
\appendix
% !TEX root = main.tex

\section{Datasets}
\label{app:datasets}

\textbf{OpenML.}
The OpenML datasets are identified by their unique ID and can be downloaded programmatically using the OpenML API (\texttt{openml.datasets.get\_dataset(ID)}).
The IDs of the 10 datasets used in this work, as well as the number of examples and features, are provided in Table \ref{tab:openml} in the main manuscript.
Categorical features are encoded using scikit-learn's label encoder (\texttt{sklearn.preprocessing.LabelEncoder}).
All of the datasets correspond to binary classification problems, with varying degrees of class imbalance. 
Stratified sampling (\texttt{sklearn.cross\_validation.StratifiedKFold}) is used to construct the outer and inner folds for the nested cross-validation.
The loss function used for training and evaluation is sample-weighted logistic loss (\texttt{sklearn.metrics.log\_loss}).

\textbf{Rossmann Store Sales.}
We download the raw data programmatically using the Kaggle API, which produces two files: \texttt{train.csv} and \texttt{store.csv}.
Both files are read into pandas data frames and the missing values are replaced with zeros.
We then follow several preprocessing steps inspired by an existing Kaggle kernel\footnote{https://www.kaggle.com/cast42/xgboost-in-python-with-rmspe-v2}.

Firstly, we merge the train data frame with the store data frame, on the \texttt{Store} column.
The resulting data frame is then sorted in ascending order by date.
We then filter the data to exclude any stores that are not open, or have 0 sales. 
Next, we perform label encoding of the three categorical variables \texttt{StoreType}, \texttt{Assortment} and \texttt{StateHoliday}.
We then extract four numeric features (month, year, day, and week of year) from the date feature.
We create a feature corresponding to the number of months since the competition was open, 
and a similar feature corresponding to how many months a promotion has been running.
We create one additional binary feature indicating whether the month is in the promotion interval. 
We then extract the \texttt{Sales} column as the labels and apply a logarithmic transformation.
After all the pre-processing steps described above, the data matrix has 20 features.

While the prediction is always performed in the logarithmic domain, when evaluating the models we transform both the labels and the model predictions back into their original domain. 
The loss function used for training and evaluation is the standard root mean-squared error (\texttt{sklearn.metrics.mean\_squared\_error}).
To create the train/validation/test split, we first extract all rows corresponding to the month of July.
The extracted rows are then split 50/50 to form the validation and test set using the \texttt{train\_test\_split} function from scikit-learn with seed 42.
The remaining rows are used for training only. The number of examples used for training, validation and test are 758762, 42788, and 42788, respectively.

\textbf{Mercari Price Suggestion.}
We download the raw data programmatically using the Kaggle API, which produces the file \texttt{train.tsv}.
We then follow several preprocessing steps inspired by an existing Kaggle kernel\footnote{https://www.kaggle.com/tsaustin/mercari-price-recommendation}.

Firstly, we remove the products with price 0.
Next, we replace missing values in the \texttt{name}, \texttt{category\_name} and \texttt{item\_description} columns with a constant string.
We then \textit{clean} these 3 columns, by 1) removing non-alpha characters, 2) converting to lower-case, and 3) applying scikit-learn's 
\texttt{CountVectorizer} with English stop-words and the maximum number of features set to 30.
Next, we perform target encoding on the \texttt{brand\_name} feature (\texttt{data`brand\_name'].map(data.groupby(`brand\_name')[`price'].mean())}) and 
we encode the \texttt{shipping} column using one-hot encoding (\texttt{pandas.get\_dummies}).
After all the pre-processing steps described above, the data matrix has 98 features.

Then, we extract the \texttt{price} column as the labels and apply a logarithmic transformation.
The root mean squared error loss function is used for training and evaluation, with labels and predictions transformed back to the original domain.
We then run L1-normalization on the rows and perform an 80/20 trainval/test split (with seed 42).
The trainval set is then split 70/30 to generate the train and validation sets (with seed 42).
The number of examples used for training, validation and test are 829729, 355599, and 296333 respectively.

\textbf{Credit Card Fraud.}
We download the raw data programatically using the Kaggle API, which produces the file \texttt{creditcard.csv}.
We extract the 31-st column as the binary labels and remove the first column \texttt{Time}.
With the remaining columns we apply scikit-learn's \texttt{StandardScaler} followed by L1-normalization of the rows.
After all the pre-processing steps described above, the data matrix has 31 features.

We then perform a stratified 75/25 trainval/test split (seed 42), followed by a 70/30 train/val split (seed 42).
The number of examples used for training, validation and test are 149523, 64082, and 71202 respectively.
Since the data is highly imbalanced, for training and evaluation we use the sample-weighted logistic loss.
The sample weights are computed using the \texttt{compute\_sample\_weight} function from scikit-learn (\texttt{sklearn.utils.class\_weight}), using the \texttt{balanced} option.

\section{Hyper-Parameters of \name{}}
\label{sec:snapboost-hp}

In the following we list the hyper-parameters of the \name{} algorithm. We highlight in \textbf{\emph{bold}} the hyper-parameters that typically require tuning when performing hyper-parameter optimization.

%\begin{table}[h]
%\centering
%\begin{tabular}{|l|c|c|c|}
%\hline
%\textbf{Hyper-parameter}    & \textbf{Type}     & \textbf{Description} \\ \hline
%\emph{num\_round}           & int               & number of boosting iterations \\
%\emph{objective}            & 'mse', 'logloss'  & loss function optimized by the boosting algorithm \\
%\emph{learning\_rate}       & float             & learning rate used by the boosting algorithm \\
%\emph{random\_state}        & int               & random seed used at training time \\
%\emph{colsample}            & float             & fraction of features to be subsampled at each boosting iteration \\
%\end{tabular}
%\caption{The hyper-parameters of the \name{} algorithm.}
%\label{tab:hp-snapboost}
%\end{table}

\begin{itemize}

    \item \textbf{\emph{num\_round}} (int): the number of boosting iterations.
    \item \emph{objective} ('mse', 'logloss'): the loss function optimized by the boosting algorithm.
    \item \textbf{\emph{learning\_rate}} (float): the learning rate of the boosting algorithm.
    \item \emph{random\_state} (int): the random seed used at training time.
    \item \textbf{\emph{colsample}} (float): the fraction of features to be subsampled at each boosting iteration.
    \item \textbf{\emph{subsample}} (float): the fraction of examples to be subsampled at each boosting iteration.
    \item \textbf{\emph{lambda\_l2}} (float): L2-regularization parameter applied to the tree leaf values.
    \item \emph{early\_stopping\_rounds} (int): the number of boosting iterations used by early stopping.
    \item \emph{base\_score} (float): the initial prediction of all examples.
    \\
    
    \item \textbf{\emph{tree\_probability}} (float): the probability of selecting a tree at a boosting iteration.
    \item \textbf{\emph{min\_max\_depth}} (int) : the minimum max\_depth of a tree in the ensemble.
    \item \textbf{\emph{max\_max\_depth}} (int): the maximum max\_depth of a tree in the ensemble.
    \item \emph{use\_histograms} (bool): whether the tree uses histogram statistics or not.
    \item \emph{hist\_nbins} (int): number of histogram bins if \emph{use\_histograms} is \texttt{True}.
    \item \emph{tree\_n\_threads} (int): the number of threads used to train the trees.
    \\
    
    \item \textbf{\emph{alpha}} (float): the regularizer of the ridge regressor. 
    \item \textbf{\emph{fit\_intercept}} (bool): whether to fit the intercept of the ridge regressor or not.
    \item \emph{ridge\_n\_threads} (int): the number of threads used to train the ridge regressor.
    \\
    
    \item \textbf{\emph{gamma}} (float): the gamma value of the Gaussian radial basis kernel.
    \item \textbf{\emph{n\_components} (c)} (int): the dimension of the randomized feature space.
    \item \emph{kernel\_n\_threads} (int): the number of threads used to compute the dataset projection onto the new randomized feature space.
    
\end{itemize}

\section{Hyper-Parameter Optimization Method: Successive Halving}
\label{app:sh}

\begin{algorithm}[t]
\begin{algorithmic}[1]
\State \textbf{Input: } initial number of configurations $n_0$ 
\State \textbf{Input: } elimination rate $\eta$
\State \textbf{Input: } minimum resource $r_{min}$
\State \textbf{Input: } number of processor cores $\texttt{num\_cores}$
\State Determine number of stages $s_{max} = \lfloor-\log_{\eta}(r_{min})\rfloor$
\State \textbf{Assert: } $n_0 \geq \eta^{s_{max}}$
\State Initialize set $C$ by sampling $n_0$ configurations at random
\For {$i=0,1,\ldots,s_{max}$}
	\State Set number of configurations in this stage: $n_i = \lfloor n_0\eta^{-i}\rfloor$
	\State Set resource in this stage:  $r_i = \eta^{i-s_{max}}$
	\State Set number of processes in this stage: $p_i = \min\left(\texttt{num\_cores}, |C|\right)$
	\State Set number of threads in this stage: $t_i = \lfloor \texttt{num\_cores} / p_i \rfloor$
	\State Populate input queue $Q_{in}$ with all configurations $c\in C$
	\ParFor {$p=0,1,\ldots,p_i$}
		\State Start new process with $t_i$ threads
		\While {$Q_{in}$ is not empty}
			\State Pull configuration $c$ from $Q_{in}$
			\State Train using fraction $r_i$ of training examples and compute validation loss $l$ \label{step:train}
			\State Push $(c,l)$ pair into output queue $Q_{out}$
		\EndWhile
	\EndParFor
	\State Sort output queue $Q_{out}$ by validation loss
	\State Update set $C$ to comprise the $n_i/\eta$ configurations with lowest validation loss
\EndFor
\State \textbf{Output: } Configuration in $C$ with lowest validation loss
\end{algorithmic}
\caption{Successive Halving \cite{pmlr-v51-jamieson16} with process-level and thread-level parallelism.}
\label{alg:sh}
\end{algorithm}

To perform hyper-parameter tuning we use the successive halving (SH) method from \cite{pmlr-v51-jamieson16}. 
SH begins by training and evaluating a large number of hyper-parameter configurations using only a small fraction of the training examples (otherwise referred to as \textit{resource}).
The configurations are then ranked according to their validation loss and only the best-performing configurations are carried forward into the next \textit{stage}, in which they are trained using a larger resource. 
This process repeats until the final stage, where all remaining configurations are trained using the maximal resource (i.e., the full train set).
The general idea is that bad configurations can be eliminated in the earlier stages, without consuming a significant number of CPU cycles.
Our implementation of SH is massively parallel and leverages both process-level parallelism (across configurations) and thread-level parallelism (within configurations).
The implementation is described in full in Algorithm \ref{alg:sh}.
 
For the OpenML benchmark we used Algorithm \ref{alg:sh} with $n_0=512$, $\eta=4$ and $r_{min}=1/4$. 
Since 3x3 nested cross-validation was used in this benchmark, the SH method was performed independently for each of the 3 outer folds. 
For each outer fold, we perform a cross-validated variant of Algorithm \ref{alg:sh} in which Step \ref{step:train} is performed across the 3 inner folds. 
Specifically, each configuration is trained and evaluated for every inner fold, and the validation loss used to rank the configurations is given as the mean across the 3 inner folds.
For the Kaggle benchmark we used Algorithm \ref{alg:sh} with $n_0=256$, $\eta=4$ and $r_{min}=1/16$.
Additionally, we also leveraged the early-stopping functionality of the boosting frameworks, so that the training of each configuration may terminate early, if it is detected that the validation loss has not improved in the last 10 boosting iterations.
Identical hyper-parameter ranges were used in both benchmarks, and are given in full in the next section. 

\section{Hyper-Parameter Search Space}
\label{app:hp-space}

In Tables~\ref{tab:hp-xgb},~\ref{tab:hp-lgb},~\ref{tab:hp-cat},~\ref{tab:hp-kt} and ~\ref{tab:hp-mix} we list the hyper-parameter ranges for XGBoost, LightGBM, CatBoost, KTBoost and \name{}, respectively.

\textbf{Maximum depth.}
LightGBM enforces a constraint on the maximum number of leaves that corresponds to a constraint on the maximum (complete) tree depth of 16. 
Furthermore, in our setup we had to further limit the maximum depth to 15 in order to avoid out-of-memory errors.
For CatBoost, we had to limit the maximum depth to 16, for the same reason.
The other three frameworks (XGBoost, \name{} and KTBoost) were able to use trees of depth up to 19 without any memory issues. 

\textbf{Ordered boosting.}
CatBoost offers both ordered boosting (\texttt{boosting\_type=Ordered}) as well as standard boosting (\texttt{boosting\_type=Plain}).
All CatBoost experiments presented in the paper were performed using the default setting of ordered boosting. 
Since CatBoost seems to be slower than the other frameworks, afterwards we re-ran the experiments using \texttt{boosting\_type=Plain} but were only able to see around a 20\% improvement in runtime. 

\textbf{KTBoost.}
For KTBoost, we were unable to use the early stopping functionality as the library generated errors.
In the main manuscript, we report the KTBoost results obtained for the Credit Card Fraud dataset. 
In the meantime, we additionally collected the KTBoost results for the Rossman Store Sales dataset: 54 hours (total tuning and evaluation time) vs. less than 30 minutes (\name{}), average test score (root mean squared error) 659.73 vs. 627.50 (\name{}).

% KTBoost rossmann results
% ,t_test,t_tune,test_score_0,test_score_1,test_score_2,test_score_3,test_score_4,test_score_5,test_score_6,test_score_7,test_score_8,test_score_9,test_score_m,test_score_s,val_score
%kt,16577.924875497818,177922.1005268097,435249.920727327,437186.1567899785,433342.61982231395,439654.1922163741,440308.30583598,438428.08168394625,431543.47600167524,428898.51613664767,434907.0389223992,433003.0018153453,435252.13099519874,3488.7908737924563,881286.7107816108

%1) for imbalanced classification tasks (e.g., the \emph{credit\_card} dataset), the library API does not provide support for training with sample weights separately provided for the train and validation datasets, and 2) for classification/regression tasks (e.g., the \emph{rossmann} dataset), enabling early stopping caused the library to crash.

\begin{table}[h!]
    \parbox[t]{.45\linewidth}{
        \centering
        \caption{XGBoost hyper-parameter ranges.}
        \label{tab:hp-xgb}
        \begin{tabular}{lccc}
            \hline
            \textbf{Hyper-parameter} & \multicolumn{1}{l}{\textbf{Min}} & \multicolumn{1}{l}{\textbf{Max}} & \multicolumn{1}{l}{\textbf{Scale}} \\ \hline
            max\_depth               & 1                                 & 19                                & Linear                              \\ %\hline
            num\_round               & 10                                & 1000                              & Linear                              \\ %\hline
            learning\_rate           & -2.5                              & -1                                & Log10                               \\ %\hline
            colsample\_bytree        & 0.5                               & 1.0                               & Linear                              \\ %\hline
            subsample                & 0.5                               & 1.0                               & Linear                              \\ %\hline
            lambda                   & -2                                & -2                                & Log10                               \\ \hline
            tree\_method             & \multicolumn{3}{c}{hist}                                                                                   \\ %\hline
            max\_bin                 & \multicolumn{3}{c}{256}                                                                                    \\ \hline
        \end{tabular}
    }
    \hfill
    \parbox[t]{.45\linewidth}{
        \centering
        \caption{LightGBM hyper-parameter ranges.}
        \label{tab:hp-lgb}
        \begin{tabular}{lccc}
            \hline
            \textbf{Hyper-parameter} & \multicolumn{1}{l}{\textbf{Min}} & \multicolumn{1}{l}{\textbf{Max}} & \multicolumn{1}{l}{\textbf{Scale}} \\ \hline
            max\_depth               & 1                                 & 15                                & Linear                              \\ %\hline
            num\_round               & 10                                & 1000                              & Linear                              \\ %\hline
            learning\_rate           & -2.5                              & -1                                & Log10                               \\ %\hline
            feature\_fraction        & 0.5                               & 1.0                               & Linear                              \\ %\hline
            bagging\_fraction        & 0.5                               & 1.0                               & Linear                              \\ %\hline
            lambda\_l2               & -2                                & -2                                & Log10                               \\ \hline
            max\_bin                 & \multicolumn{3}{c}{256}                                                                                    \\ \hline
        \end{tabular}
    }
\end{table}

\begin{table}[h!]
    \parbox[t]{.45\linewidth}{
        \centering
        \caption{CatBoost hyper-parameter ranges.}
            \label{tab:hp-cat}
        \begin{tabular}{lccc}
            \hline
            \textbf{Hyper-parameter} & \multicolumn{1}{l}{\textbf{Min}} & \multicolumn{1}{l}{\textbf{Max}} & \multicolumn{1}{l}{\textbf{Scale}} \\ \hline
            max\_depth               & 1                                 & 16                                & Linear                              \\ %\hline
            n\_estimators            & 10                                & 1000                              & Linear                              \\ %\hline
            learning\_rate           & -2.5                              & -1                                & Log10                               \\ %\hline
            subsample                & 0.5                               & 1.0                               & Linear                              \\ %\hline
            l2\_leaf\_reg            & -2                                & -2                                & Log10                               \\ \hline
            max\_bin                 & \multicolumn{3}{c}{256}                                                                                    \\ %\hline
	    boosting\_type           & \multicolumn{3}{c}{Ordered}                                                                                \\ %\hline
            bootstrap\_type          & \multicolumn{3}{c}{MVS}                                                                                    \\ %\hline
            sampling\_frequency      & \multicolumn{3}{c}{PerTree}                                                                                \\ %\hline
            grow\_policy             & \multicolumn{3}{c}{SymmetricTree}                                                                          \\ \hline
        \end{tabular}
    }
    \hfill
    \parbox[t]{.45\linewidth}{
        \centering
        \caption{KTBoost hyper-parameter ranges.}
        \label{tab:hp-kt}
        \begin{tabular}{lccc}
        \hline
        \multicolumn{1}{l}{\textbf{Hyper-parameter}} & \multicolumn{1}{l}{\textbf{Min}} & \multicolumn{1}{l}{\textbf{Max}} & \multicolumn{1}{l}{\textbf{Scale}} \\ \hline
        max\_depth                                     & 1                                 & 19                                & Linear                              \\ %\hline
        n\_estimators                                  & 10                                & 1000                              & Linear                              \\ %\hline
        learning\_rate                                 & -2.5                              & -1                                & Log10                               \\ %\hline
        subsample                                      & 0.5                               & 1.0                               & Linear                              \\ %\hline
        max\_features                                  & 0.5                               & \multicolumn{1}{l}{1.0}          & \multicolumn{1}{l}{Linear}         \\ %\hline
        theta                                          & -1.5                              & 1.5                               & Log10                               \\ %\hline
        alphaReg                                       & -6                                & \multicolumn{1}{l}{3}            & \multicolumn{1}{l}{Log10}          \\ %\hline
        n\_components                                  & 1                                 & \multicolumn{1}{l}{100}          & \multicolumn{1}{l}{Linear}         \\ \hline
        update\_step                                   & \multicolumn{3}{c}{newton}                                                                                 \\ %\hline
        nystroem                                       & \multicolumn{3}{c}{True}                                                                                   \\ %\hline
        base\_learner                                  & \multicolumn{3}{c}{combined}                                                                               \\ \hline
        \end{tabular}
    }
\end{table}

\begin{table}[h!]
    \centering
    \caption{\name{} hyper-parameter ranges}
    \label{tab:hp-mix}
    \begin{tabular}{lccc}
    \hline
    \textbf{Hyper-parameter} & \textbf{Min} & \textbf{Max} & \textbf{Scale} \\ \hline
    num\_round               & 10           & 1000         & Linear         \\ %\hline
    min\_max\_depth          & 1            & 19           & Linear         \\ %\hline
    max\_max\_depth          & 1            & 19           & Linear         \\ %\hline
    learning\_rate           & -2.5         & -1           & Log10          \\ %\hline
    subsample                & 0.5          & 1.0          & Linear         \\ %\hline
    colsample                & 0.5          & 1.0          & Linear         \\ %\hline
    lambda\_l2               & -2           & -2           & Log10          \\ %\hline
    tree\_probability        & 0.9          & 1.0          & Linear         \\ %\hline
    fit\_intercept           & 0 (False)    & 1 (True)     & Linear         \\ %\hline
    alpha                    & -6           & -3           & Log10          \\ %\hline
    gamma                    & -3           & 3            & Log10          \\ %\hline
    n\_components            & 1            & 100          & Linear         \\ \hline
    hist\_nbins              & \multicolumn{3}{c}{256}                     \\ \hline
    \end{tabular}
\end{table}

\section{NODE versus \name Benchmark}
\label{sec:node-mixboost}

In this section, we compare \name with Neural Oblivious Decision Ensembles (NODE)~\cite{popov2019neural}.
NODE constructs deep networks of \textit{soft} decision trees that can be trained using end-to-end back-propagation.

\textbf{Datasets.} For this benchmark we used 6 regression datasets as shown in Table~\ref{tab:uci}. 
These datasets have approximately 10K examples and 20 features on average. 
We chose to use these relatively small datasets since, as we will see, training NODE is fairly slow.
Firstly, we manually downloaded the data using the links provided in Table~\ref{tab:uci}. 
For some datasets, we concatenated the provided train and test data files (\emph{ailerons}, \emph{elevators}, \emph{puma32H}, and \emph{bank8FM}). 
We used the concatenated matrices as input to the train/validation/test splitting during hyper-parameter optimization. 
As labels, we used column 41 for \emph{ailerons}, column 20 for \emph{parkinsons}, column 17 for \emph{navalT}, column 19 for \emph{elevators}, column 33 for \emph{puma32h}, and column 9 for \emph{bank8FM} (column indices being 1-based). 
We do not perform additional data preprocessing.

%ailerons (concatenated): 1-40, 41
%parkinsons (-): 1-19, 20
%navalT (-): 1-16, 17
%elevators (concatenated): 1-18, 19
%puma32H (concatenated): 1-32, 33
%bank8FM (concatenated): 1-8, 9

\textbf{Infrastructure.} 
The results in this section were obtained using a single-socket server with an 8-core Intel(R)Xeon(R) CPU E5-2630 v3 CPU, @2.40GHz, 2 threads per core, 64 GiB RAM, 2 NVIDIA GTX 1080 TI GPUs, running Ubuntu 16.04. 
%We used this GPU-enabled machine in order to let NODE exploit GPU acceleration.
We use NODE commit \texttt{3bae6a8a63f0205683270b6d566d9cfa659403e4} and PyTorch 1.4.0.

\textbf{Hyper-parameter optimization method.} 
To tune the hyper-parameters of NODE and \name{}, we used the optimization method described in Algorithm~\ref{alg:sh} with $n_0=1000$, $\eta=4$ and $r_{min}=1/20$. 
We tuned \name{} on the CPU using 16 single-threaded processes in parallel (\texttt{num\_cores=16}). 
NODE was tuned sequentially, one hyper-parameter configuration at a time, using both available GPUs. 
It was necessary to both GPUs since NODE crashed with out-of-memory errors when using only one.

In this benchmark, we used 2x2 nested cross-validation. 
The SH method was performed independently for each of the 2 outer folds. 
For each outer fold, we performed a cross-validated variant of Algorithm \ref{alg:sh} in which Step \ref{step:train} was performed across the 2 inner folds. 
Specifically, each configuration was trained and evaluated for every inner fold, and the validation loss used to rank the configurations was given as the mean across the 2 inner folds. 
The training and evaluation loss used in this benchmark was the root mean squared error (RMSE). 

\textbf{Hyper-parameter search space.} 
Tables~\ref{tab:hp-mix-node} and~\ref{tab:hp-node-node} show the hyper-parameter ranges used in this benchmark. 
NODE's layer dimension is computed as $\texttt{layer\_dim} = \lceil \texttt{total\_trees} / \texttt{num\_layers} \rceil$. 
Other NODE parameter settings: \texttt{nus=(0.7,1.0}, \texttt{betas=(0.95, 0.998)}, \texttt{optimizer=QHAdam}, \texttt{epochs=100}, and \texttt{batch\_size=min(int(dataset.shape[0]/2), 512)}.

\begin{table}[h!]
    \parbox[t]{.45\linewidth}{
        \centering
        \caption{\name{} hyper-parameter ranges.}
        \label{tab:hp-mix-node}
        \begin{tabular}{lccc}
            \hline
            \textbf{Hyper-parameter} & \multicolumn{1}{l}{\textbf{Min}} & \multicolumn{1}{l}{\textbf{Max}} & \multicolumn{1}{l}{\textbf{Scale}} \\ \hline
                num\_round               & 64           & 2048         & Linear         \\ %\hline
                min\_max\_depth          & 1            & 8            & Linear         \\ %\hline
                max\_max\_depth          & 1            & 8            & Linear         \\ %\hline
                learning\_rate           & -3           & 0            & Log10          \\ %\hline
                subsample                & 0.5          & 1.0          & Linear         \\ %\hline
                colsample                & 0.5          & 1.0          & Linear         \\ %\hline
                lambda\_l2               & -2           & -2           & Log10          \\ %\hline
                tree\_probability        & 0.9          & 1.0          & Linear         \\ %\hline
                fit\_intercept           & 0 (False)    & 1 (True)     & Linear         \\ %\hline
                alpha                    & -6           & 3            & Log10          \\ %\hline
                gamma                    & -3           & 3            & Log10          \\ %\hline
                n\_components            & 1            & 100          & Linear         \\ \hline
                hist\_nbins              & \multicolumn{3}{c}{256}                     \\ \hline
        \end{tabular}
    }
    \hfill
    \parbox[t]{.45\linewidth}{
        \centering
        \caption{NODE hyper-parameter ranges.}
        \label{tab:hp-node-node}
        \begin{tabular}{lccc}
            \hline
            \textbf{Hyper-parameter} & \multicolumn{1}{l}{\textbf{Min}} & \multicolumn{1}{l}{\textbf{Max}} & \multicolumn{1}{l}{\textbf{Scale}} \\ \hline
            num\_layers              & 1                                 & 8                                 & Linear                              \\ %\hline
            total\_trees             & 64                                & 2048                              & Linear                              \\ %\hline
            depth                    & 1                                 & 8                                 & Linear                              \\ %\hline
            tree\_dim                & 2                                 & 3                                 & Linear                              \\ \hline
        \end{tabular}
    }
\end{table}

\textbf{Experimental results.} Table~\ref{tab:uci} shows the result of the benchmark. 
We used 6 publicly-available datasets: ailerons~\cite{ailerons}, parkinsons~\cite{parkinsons,parkinsons-paper}, navalT~\cite{navalT,navalT-paper}, elevators~\cite{elevators}, bank8FM~\cite{bank8FM}, and puma32h~\cite{puma32h}.
The table includes information about the datasets' characteristics, as well as the test RMSE (averaged over the 2 outer folds), and total experimental time for both NODE and \name{}. 
The time is reported in hours. 
\name{} achieves a lower test RMSE than NODE on 4 datasets, whereas NODE wins on the remaining 2 datasets. 
In terms of experimental time, \name{} is on average approximately 160 times faster than NODE.

\begin{table}[]
\centering
\caption{NODE vs. \name{} Benchmark.}
\label{tab:uci}
\resizebox{\textwidth}{!}{%
\begin{tabular}{@{}lccccccc@{}}
	\hline
                                                                                            & \multicolumn{1}{l}{} & \multicolumn{1}{l}{} & \multicolumn{2}{c}{\textbf{RMSE (Test)}}     & \multicolumn{2}{c}{\textbf{Time (hours)}} & \multicolumn{1}{l}{}   \\
\textbf{Name}                                                                               & \textbf{Rows}        & \textbf{Features}    & \textbf{NODE}     & \textbf{\name{}}  & \textbf{NODE}      & \textbf{\name{}}     & \textbf{Speed-up} \\ \hline
ailerons                                                                    & 13750                & 40                   & 0.000204          & \textbf{0.000157} & 38.25            & 0.34               & 112.5                  \\
parkinsons & 5875                 & 19                   & 0.001718          & \textbf{0.000868} & 22.30            & 0.15               & 150                  \\
navalT                                                           & 11934                & 16                   & 0.006941          & \textbf{0.000631} & 55.19            & 0.20               & 275.9                 \\
elevators                                                                  & 16599                & 18                   & 0.005099          & \textbf{0.002074} & 32.83            & 0.24               & 136.7                  \\
bank8FM                                                                      & 8192                 & 8                    & \textbf{0.028717} & 0.031298          & 27.45            & 0.11               & 249.5                 \\
puma32h                                                                      & 8192                 & 32                   & \textbf{0.006424} & 0.007629          & 23.70            & 0.39               & 60.7                   \\ \hline
\end{tabular}%
}
\end{table}

\section{Reformulation of Algorithm \ref{alg:hnbm} as Coordinate Descent}
\label{app:coord}

The definition of the $\mathcal{F}$ given in \eqref{eq:domain} dictates that any function $f\in\mathcal{F}$ can be expressed a weighted sum over functions belong to the base hypothesis class $\mathcal{H}$.
Furthermore, by Assumption \ref{assumption:norm}, every function in $\mathcal{H}$ can be expressed as a scalar multiplied by one of the functions belonging to finite set $\mathcal{\bar{H}}$.
Thus, every $f\in\mathcal{F}$ has an equivalent representation as a weighted sum over the functions $b_j\in\mathcal{\bar{H}}$:
\begin{equation*}
	f(x_i) = \sum_{j=1}^{|\mathcal{\bar{H}}|}\beta_j b_j(x_i),
\end{equation*}
where $\beta\in\mathbb{R}^{|\mathcal{\bar{H}}|}$ and typically the vast majority of the coefficients $\beta_j$ are zero.
Next, we introduce the matrix $B\in\mathbb{R}^{n\times|\mathcal{\bar{H}}|}$, with entries given by $B_{i,j} = b_j(x_i)$.
Given this definition, a given function $f\in\mathcal{F}$ evaluated at $x_i$ can be expressed:
\begin{equation*}
	f(x_i) = \sum_{j=1}^{|\mathcal{\bar{H}}|}\beta_j B_{i,j} = B_i\beta,
\end{equation*}
where $B_i\in\mathbb{R}^{1\times|\mathcal{\bar{H}}|}$ denotes the $i$-th row of $B$.
Thus minimization \eqref{eq:obj} over domain \eqref{eq:domain} is equivalent to minimizing the following objective function over $\beta\in\mathbb{R}^{|\mathcal{\bar{H}}|}$:
\begin{equation*}
	L(\beta) = \sum_{i=1}^n l(y_i, B_i\beta)
\end{equation*}
The optimal coordinate to update at the $m$-th iteration, given randomly chosen subclass index $u_m$, is given by:
\begin{align}
	j_m &= \argmin_{j\in I(u_m)}\left[\min_{\sigma\in\mathbb{R}} L(\beta^{m-1} + \sigma e_j) \right]\nonumber \\
	&= \argmin_{j\in I(u_m)}\left[\min_{\sigma\in\mathbb{R}} \sum_{i=1}^n l(y_i, B_i\beta^{m-1} + \sigma B_{i,j})\right] \nonumber \\
	&\approx \argmin_{j\in I(u_m)}\left[\min_{\sigma\in\mathbb{R}} \sum_{i=1}^n \left(l(y_i, B_i\beta^{m-1}) + g_i \sigma B_{i,j} + \frac{h_i}{2}\sigma^2 B_{i,j}^2 \right)\right]\nonumber\\
	&=  \argmin_{j\in I(u_m)}\left[\min_{\sigma\in\mathbb{R}}\sum_{i=1}^n h_i \left(-\frac{g_i}{h_i} - \sigma B_{i,j}\right)^2 \right],\label{eq:mini_coord}
\end{align}
where the approximation is obtained by taking the second-order Taylor expansion of $l(y_i, B_i\beta^{m-1} + \sigma B_{i,j})$ around $l(y_i, B_i\beta^{m-1})$,  
with expansion coefficients given by $g_i=l'(y_i, B_i\beta^{m-1})$ and $h_i=l''(y_i, B_i\beta^{m-1})$.
Note that this optimization problem is directly equivalent to \eqref{eq:regression} in the original formulation of Algorithm \ref{alg:hnbm}.
For a fixed coordinate $j$, the inner minimization over $\sigma$ has a closed-form solution:
\begin{equation}
	\sigma^*_j = -\frac{\sum_i g_i B_{i,j}}{\sum_i h_i B_{i,j}^2} = -\frac{\nabla_j L(\beta^{m-1})}{\nabla_j^2 L(\beta^{m-1})},\label{eq:sigma_explicit}
\end{equation}
where we have used two identities that link the first and second-order derivatives of $L(\beta)$ to the coefficients $g_i$ and $h_i$ as follows:
\begin{align}
	\nabla_j L(\beta^{m-1}) &= \frac{\partial}{\partial\beta^{m-1}_j} \left(\sum_{i=1}^n l(y_i, B_i\beta^{m-1})\right) = \sum_{i=1}^n g_i B_{i,j} \label{eq:nabla1} \\
	\nabla_j^2 L(\beta^{m-1}) &= \frac{\partial^2}{\partial(\beta^{m-1}_j)^2} \left(\sum_{i=1}^n l(y_i, B_i\beta^{m-1})\right) = \sum_{i=1}^n h_i B_{i,j}^2 \label{eq:nabla2}.
\end{align}
Now, by plugging \eqref{eq:sigma_explicit} into \eqref{eq:mini_coord} we have:
\begin{align}
j_m &= \argmin_{j\in I(u_m)}\left[\sum_{i=1}^n h_i \left(-\frac{g_i}{h_i} + \frac{\nabla_j L(\beta^{m-1})}{\nabla_j^2 L(\beta^{m-1})} B_{i,j}\right)^2\right] \nonumber \\
&= \argmin_{j\in I(u_m)}\left[ -2\frac{\nabla_j L(\beta^{m-1})}{\nabla_j^2 L(\beta^{m-1})}\sum_{i=1}^n g_i B_{i,j} + \left(\frac{\nabla_j L(\beta^{m-1})}{\nabla_j^2 L(\beta^{m-1})}\right)^2\sum_{i=1}^n h_i B_{i,j}^2 \right]\nonumber\\
&= \argmin_{j\in I(u_m)}\left[-\frac{\left(\nabla_j L(\beta^{m-1})\right)^2}{\nabla_j^2 L(\beta^{m-1})}\right] 
= \argmax_{j\in I(u_m)}\left[\left|\frac{\nabla_j L(\beta^{m-1})}{\sqrt{\nabla_j^2 L(\beta^{m-1})}}\right|\right]\nonumber,
\end{align}
where in the third equality we have again used \eqref{eq:nabla1} and \eqref{eq:nabla2}.

\section{Proof of Lemma \ref{lemma:exp}}
\label{app:lemma}
This proof is analogous to Proposition 4.3 in \cite{lu2018randomized}, adapted to use the norm induced by $\Phi$, as well as the second-derivative information. 
From the statement of the Lemma, we have the following definition of $\Gamma(\beta)$ for $j\in\left[|\mathcal{\bar{H}}|\right]$:
\begin{equation*}
	\Gamma_j(\beta) = \frac{\nabla_{j}L(\beta)}{\sqrt{\nabla_{j}^2L(\beta)}}
\end{equation*}
Now, given the definition of $j_m$ in \eqref{eq:update_coord}, we have:
\begin{equation*}
	\mathbb{E}_m\left[ \Gamma_{j_m}(\beta^{m-1})^2 \right] = \sum_{k=1}^K \phi_k \max_{j\in I(k)}\Gamma_j(\beta^{m-1})^2 = \sum_{j=1}^{|\mathcal{\bar{H}}|} \lambda_j \Gamma_j(\beta^{m-1})^2,
\end{equation*}
where $\lambda_j$ is defined as follows:
\begin{equation*}
	\lambda_j = \left.
\begin{cases}
    \phi_1, & \text{if } j = \argmax_{j\in I(1)}\Gamma_j(\beta^{m-1})^2 \\
	\phi_2, & \text{if } j = \argmax_{j\in I(2)}\Gamma_j(\beta^{m-1})^2 \\
	\vdots \\
	\phi_K, & \text{if } j = \argmax_{j\in I(K)}\Gamma_j(\beta^{m-1})^2 \\
    0, & \text{otherwise}.
 \end{cases}\right.
\end{equation*}
Now, by noting that $\sum_{j}\lambda_j = 1$ and $\lambda_j\geq0$ and applying the Cauchy-Schwarz inequality:
\begin{align*}
	\mathbb{E}_m\left[ \Gamma_{j_m}(\beta^{m-1})^2 \right] &= \left(\sum_j\lambda_j\right)\left(\sum_j \lambda_j\Gamma_j(\beta^{m-1})^2\right) \nonumber \\
	&\geq \left(\sum_j\lambda_j|\Gamma_j(\beta^{m-1})|\right)^2 = ||\Gamma(\beta^{m-1})||_\Phi^2,
\end{align*}
where the last equality uses the definition of the $\Phi$-norm from Definition \ref{defn:mca}. \qed

\section{Proof of Theorem \ref{thm:converge}}
\label{app:theorem}

From the update rule \eqref{eq:update} and equations \eqref{eq:update_magnitude} and \eqref{eq:update_coord} we have:
\begin{align}
	L(\beta^{m}) &= \sum_{i=1}^n l\left(y_i, B_i\beta^{m-1} - \epsilon\left(\frac{\nabla_{j_m}L(\beta^{m-1})}{\nabla_{j_m}^2L(\beta^{m-1})}\right)B_{i,j_m}\right)\nonumber\\
	&\leq \sum_{i=1}^n l(y_i, B_i\beta^{m-1}) - \epsilon\left(\frac{\nabla_{j_m}L(\beta^{m-1})}{\nabla_{j_m}^2L(\beta^{m-1})}\right)B_{i,j_m}g_i \nonumber \\
	&+ \frac{\epsilon^2}{2}\left(\frac{\nabla_{j_m}L(\beta^{m-1})}{\nabla_{j_m}^2L(\beta^{m-1})}\right)^2 B_{i,j_m}^2 l''(y_i, z_i),\label{eq:mvt}
\end{align}
where the existence of the sequence $z_i$ are guaranteed by the Mean Value Theorem. 
Now, applying Assumption \ref{assumption:convex} and Assumption \ref{assumption:lipschitz} we have for all $i\in[n]$:
\begin{equation}
	\frac{l''(y_i, z_i)}{l''(y_i, B_i\beta^{m-1})} \leq \frac{S}{\mu} \implies l''(y_i, z_i) \leq \frac{S}{\mu} h_i,
\end{equation}
where we recall that $l''(y, B_i\beta^{m-1})=h_i$. 
Plugging into \eqref{eq:mvt} we have:
\begin{align}
	L(\beta^{m}) &\leq L(\beta^{m-1}) - \epsilon \left(\frac{\nabla_{j_m}L(\beta^{m-1})}{\nabla_{j_m}^2L(\beta^{m-1})}\right) \sum_{i=1}^n B_{i,j_m}g_i + \frac{\epsilon^2}{2}\frac{S}{\mu} \left(\frac{\nabla_{j_m}L(\beta^{m-1})}{\nabla_{j_m}^2L(\beta^{m-1})}\right)^2 \sum_{i=1}^n B_{i,j_m}^2 h_i \nonumber \\
	&= L(\beta^{m-1}) - \Gamma_{j_m}(\beta^{m-1})^2 \left(\epsilon - \frac{\epsilon^2}{2}\frac{S}{\mu}\right)
	= L(\beta^{m-1}) - \frac{\mu}{2S}\Gamma_{j_m}(\beta^{m-1})^2,\label{eq:upper1}
\end{align}
where $\Gamma_j(\beta)$ is defined as in Lemma \ref{lemma:exp} and in the final step we have set the learning rate to be $\epsilon=\frac{\mu}{S}$. 

Now we take the expectation of both sides of \eqref{eq:upper1} with respect to the $m$-th iteration to attain:
\begin{align}
	\mathbb{E}_m\left[L(\beta^{m})\right] &\leq L(\beta^{m-1}) - \frac{\mu}{2S}\mathbb{E}_m\left[\Gamma_{j_m}(\beta^{m-1})^2\right]\nonumber\\
	&\leq L(\beta^{m-1}) - \frac{\mu}{2S} \left\|\Gamma(\beta^{m-1})\right\|_{\Phi}^2 \nonumber \\
	& = L(\beta^{m-1}) - \frac{\mu}{2S}\left(\sum_{k=1}^K\phi_k\max_{j\in I(k)}\left| \frac{\nabla_{j}L(\beta^{m-1})}{\sqrt{\sum_{i=1}^n h_i B_{i,j}^2}} \right|\right)^2 \nonumber \\
	&\leq L(\beta^{m-1}) - \frac{\mu}{2S^2}\left(\sum_{k=1}^K\phi_k\max_{j\in I(k)}\left| \frac{\nabla_{j}L(\beta^{m-1})}{\sqrt{\sum_{i=1}^n B_{i,j}^2}} \right|\right)^2\nonumber \\
	&= L(\beta^{m-1}) - \frac{\mu}{2S^2}\left(\sum_{k=1}^K\phi_k\max_{j\in I(k)}\left| \nabla_{j}L(\beta^{m-1}) \right|\right)^2\nonumber\\
	&= L(\beta^{m-1}) - \frac{\mu}{2S^2} ||\nabla L(\beta^{m-1})||_\Phi^2\label{eq:app-upper}
\end{align}
where the second inequality follows from Lemma \ref{lemma:exp}, the third inequality follows from Assumption \ref{assumption:lipschitz}, and the penultimate equality follows due to Assumption \ref{assumption:norm}.

We then apply directly apply Proposition 4.4 and 4.5 from \cite{lu2018randomized} (which in turn rely on Assumption \ref{assumption:convex}) to obtain the following lower bound:
\begin{equation}
	\left\|\nabla L(\beta^{m-1})\right\|_{\Phi}^2 \geq 2\mu\Theta^2\left(L(\beta^{m-1})-L(\beta^*)\right),
	\label{app:lower}
\end{equation}
where $\beta^*$ is the vector that minimizes $L(\beta)$. 
Now subtracting $L(\beta^*)$ from both sides of \eqref{eq:app-upper} and applying \eqref{app:lower} we have:
\begin{align*}
	\mathbb{E}_m\left[L(\beta^{m}) - L(\beta^*)\right] &\leq L(\beta^{m-1}) - L(\beta^*) - \frac{\mu}{2S^2} ||\nabla L(\beta^{m-1})||_\Phi^2 \nonumber \\
	&\leq \left(L(\beta^{m-1}) - L(\beta^*)\right) \frac{\mu^2}{S^2}\Theta^2
\end{align*}
The proof is furnished by following a telescopic argument. \qed

\end{document}